\documentclass{article}

\usepackage{arxiv}
\usepackage{authblk}

\usepackage[utf8]{inputenc} % allow utf-8 input
\usepackage[T1]{fontenc}    % use 8-bit T1 fonts
\usepackage{hyperref}       % hyperlinks
\usepackage{url}            % simple URL typesetting
\usepackage{booktabs}       % professional-quality tables
\usepackage{amsfonts}       % blackboard math symbols
\usepackage{nicefrac}       % compact symbols for 1/2, etc.
\usepackage{microtype}      % microtypography
\usepackage{lipsum}
\usepackage{graphicx}
\graphicspath{ {./images/} }

\usepackage{graphicx}
\usepackage{natbib}
\usepackage{doi}

\usepackage{microtype}
\usepackage{graphicx}
\usepackage{subfigure}
\usepackage{booktabs} % for professional tables

\usepackage{booktabs}
\usepackage{graphicx} % 用于 resizebox
\usepackage{multirow} % 如果需要合并行（可选，下面代码未用到但常备）

\usepackage{hyperref}

\usepackage{url}
\usepackage{amsmath}
\usepackage{amssymb}
\usepackage{mathtools}
\usepackage{amsthm}

\usepackage{xcolor}
\usepackage{framed}

\definecolor{shadecolor}{gray}{0.95}

\makeatletter
\renewenvironment{proof}[1][\proofname]{%
  \begin{shaded} 
  \par
  \pushQED{\qed}%
  \normalfont \topsep6\p@\@plus6\p@ \relax
  \trivlist
  \item[\hskip\labelsep
        \bfseries % (B) 你的要求：Proof 字体加粗
    #1\@addpunct{.}]\ignorespaces
}{%
  \popQED\endtrivlist\@endpefalse
  \end{shaded} 
}
\makeatother
\usepackage[capitalize,noabbrev]{cleveref}

\usepackage{graphicx}
\usepackage{subcaption}
\usepackage{cleveref} % 建议配合使用，方便引用

\usepackage[T1]{fontenc}

\theoremstyle{plain}
\newtheorem{theorem}{Theorem}[section]

\newtheorem{definition}[theorem]{Definition}
\newtheorem{assumption}[theorem]{Assumption}
\newtheorem{lemma}[theorem]{Lemma}
\theoremstyle{remark}

\title{On the Spectral Flattening of Quantized Embeddings}

\author[1]{Junlin Huang}
\author[2]{Wenyi Fang}
\author[3]{Zhenheng Tang}
\author[2]{Yuxin Wang}
\author[1]{Xueze Kang}
\author[2]{Yang Zheng }
\author[3]{\authorcr Bo Li}
\author[1]{Xiaowen Chu}
\date{}
\renewcommand{\today}{}
\affil[1]{%
    The Hong Kong University of Science and Technology (Guangzhou)
}
\affil[2]{Huawei Technologies Co., Ltd}
\affil[3]{The Hong Kong University of Science and Technology}

\begin{document}
\thispagestyle{fancy} % 确保当前页应用样式
\fancyhead[R]{}       % 将右上角（Right）内容设为空
\maketitle
\begin{abstract}
Training Large Language Models (LLMs) at ultra-low precision is critically impeded by instability rooted in the conflict between discrete quantization constraints and the intrinsic heavy-tailed spectral nature of linguistic data. By formalizing the connection between Zipfian statistics and random matrix theory, we prove that the power-law decay in the singular value spectra of embeddings is a fundamental requisite for semantic encoding. We derive theoretical bounds showing that uniform quantization introduces a noise floor that disproportionately truncates this spectral tail, which induces spectral flattening and a strictly provable increase in the stable rank of representations. Empirical validation across diverse architectures including GPT-2 and TinyLlama corroborates that this geometric degradation precipitates representational collapse.  This work not only quantifies the spectral sensitivity of LLMs but also establishes spectral fidelity as a necessary condition for stable low-bit optimization.
\end{abstract}

% keywords can be removed

\section{Introduction}
The escalating computational demands of training Large Language Models (LLMs) have mandated a paradigm shift towards low-bit quantization to sustain scalability. In recent years, the community has witnessed a successful migration from standard \texttt{FP32} and \texttt{BF16} formats to \texttt{FP8} training~\cite{fp8,fp8-lm,train_inf_fp8}, realizing substantial gains in memory efficiency and throughput without sacrificing convergence stability. Currently, the research frontier is pivoting further toward 4-bit precision. 
Notably, NVIDIA's technical report on the Blackwell architecture reveals that the 
emerging \texttt{NVFP4} format reduces memory consumption by $1.8\times$, 
enabling the B200 GPU to accelerate General Matrix Multiplications 
(GeMM) by more than $5\times$ compared to the \texttt{FP8} throughput 
of the Hopper-based H100~\cite{nvidia2024blackwell}, underscoring the 
transformative potential of \texttt{FP4} training for next-generation 
foundation models~\cite{alvarez2024nvfp4}.

However, extending the training frontier to FP4 presents distinct theoretical hurdles that cannot be resolved by merely extrapolating previous successes. As precision descends to 4 bits, the representation imposes exponentially tighter constraints on dynamic range and resolution. This introduces a fundamental conflict with the inherently high-dynamic-range distributions of parameters, activations, and gradients characteristic of LLMs. Crucially, these distributional properties are not arbitrary artifacts; rather, they are intrinsic manifestations of the low-rank hypothesis of deep learning~\cite{aghajanyan2020, martin2021}. While the weight matrices of LLMs are nominally full-rank, they exhibit a heavy-tailed spectral structure—governed by Zipfian laws~\cite{Zipf}—wherein a few dominant singular values capture the bulk energy, while a long, fragile tail encodes fine-grained semantic information.

% We found that the primary peril of ultra-low-bit quantization, such as \texttt{FP4}, lies not merely in the introduction of scalar error, but in the structural degradation of this spectral tail. When the quantization noise floor encroaches upon the magnitude of these tail singular values, the model undergoes spectral flattening, effectively truncating the subtle feature subspaces essential for complex reasoning. Existing analyses often overlook this critical interplay between the discrete constraints of quantization and the continuous statistical laws of natural language. In this work, we bridge this gap by establishing a rigorous theoretical framework to quantify this spectral feature. The specific contributions are as follows:

We argue that the primary peril of ultra-low-bit quantization, such as \texttt{FP4}, extends beyond simple scalar approximation error; it resides in the structural degradation of the spectral tail. When the quantization noise floor encroaches upon the magnitude of these delicate tail singular values, the model undergoes spectral flattening. Crucially, while a higher stable rank is often associated with beneficial isotropy in representation learning, we prove that in the context of quantization, this increase signifies a pathological loss of discriminative capacity. This phenomenon effectively truncates the subtle feature subspaces essential for complex reasoning, leading to what we formalize as representational collapse. Existing analyses often overlook this critical interplay between the discrete constraints of quantization and the continuous Zipfian laws governing natural language. In this work, we bridge this gap by establishing a rigorous theoretical framework to quantify this feature:

\paragraph{Contribution 1:} We formalize the connection between the statistical properties of natural language and the spectral structure of linear layers. Consider a linear layer $\mathbf{Y} = \mathbf{XW}$, where $\mathbf{X}$ represents input token embeddings and $\mathbf{W}$ is the weight matrix. Based on empirical studies showing that word frequencies in large-scale corpora follow a Zipfian distribution, we prove that this distribution dictates a power-law decay in the singular values of $\mathbf{X}$ and the weight gradient $\nabla_{\mathbf{W}}=\mathbf{X}^T\mathbf{G}$, where $\mathbf{G}$ is the gradient of the loss with respect to the output. Specifically, we derive the following bounds for the top-$r$ singular values:
    \begin{equation}
        \sigma_k(\mathbf{X}) \leq C \cdot k^{-\alpha/2}, \quad \sigma_k(\nabla_{\mathbf{W}}) \leq (C M) \cdot k^{-\alpha/2},
    \end{equation}
    where $1 \leq k \leq r$, $C$ is a constant, $\sigma_k(\cdot)$ is the $k$-th largest singular value, $\alpha$ is the decay exponent, and $\|\mathbf{G}\|_2\leq M$.

\paragraph{Contribution 2:} We demonstrate
that block-wise symmetric uniform quantization,
the mechanism widely adopted in modern low-precision formats
like MXFP4~\cite{mxfp4} and NVFP4~\cite{nvfp4}, effectively alter the spectral properties of $\mathbf{X}$ and $\nabla_{\mathbf{W}}$. We prove that quantization disrupts the aforementioned power-law decay, increasing the stable rank of the representations.

The remainder of this paper is organized as follows. Section \ref{sec:preliminaries} outlines the statistical foundations. Section \ref{Sec:low-rank} derives the heavy-tailed spectral structure of LLMs using random matrix theory. Section \ref{sec:quantization} proves that quantization truncates this spectral tail and increases the stable rank. Section \ref{sec:experiments} provides empirical verification of these theoretical predictions across diverse model architectures.  Section \ref{sec:conclusion} concludes the paper and outlines directions for future work.

\section{Preliminaries and Theoretical Foundations}
\label{sec:preliminaries}
This section establishes notations and probabilistic frameworks necessary for analysis. We begin by defining the asymptotic behaviors used to describe high-dimensional limits. Subsequently, we formalize the statistical properties of language data and the geometric structure of embedding spaces. Finally, we introduce critical tools from matrix perturbation theory and concentration inequalities that allow us to bound the spectral impact of quantization noise. The supplementary definitions are in appendix Section~\ref{sec:supp_def}.

\subsection{Asymptotic Notations}

To  analyze the spectral properties of large-scale models, we consider the regime where the dimension $d \to \infty$. We define the following asymptotic relations to characterize the convergence behavior of sequences.

\begin{definition}[Asymptotic Equivalence]
\label{AsymEqui}
Define two sequences $A(d)$ and $B(d)$  to be asymptotically equivalent, denoted as $A(d) \sim B(d)$, if and only if their ratio converges to 1 as $d \to \infty$: 
\begin{equation}
A(d) \sim B(d) \iff \lim_{d \to \infty} \frac{A(d)}{B(d)} = 1.
\end{equation}
Or equivalently, the relative error tends to 0:
\begin{equation}
\lim_{d \to \infty} \left| \frac{A(d) - B(d)}{B(d)} \right| = 0.
\end{equation}
\end{definition}

\begin{definition}[Asymptotic Upper Bound]
\label{AsymBound}
For two sequences $A(d)$ and $B(d)$ that depend on $d$, $B(d)$ is the asymptotic upper bound of $A(d)$ if and only if there exists $C>0, d_0\in\mathbb{R}$ such that for all $d>d_0$ the following holds: 
\begin{equation}
|A(d)|\leq C\cdot |B(d)|,
\end{equation}
which is denoted as $A(d)=\mathcal{O}(B(d))$.
\end{definition}

\subsection{Statistical Structure of Language Representations}

To analyze the origin of the heavy-tailed spectra in LLMs, we first model the input data distribution that drives the training dynamics. We establish the spectral properties of the data embedding matrix by two widely accepted assumptions regarding token frequency and embedding geometry. 

\begin{assumption}[\cite{Zipf}]
\label{Zipf_law}
Assume the data stream is generated from a vocabulary $\mathcal{V} = \{w_1, w_2, \dots, w_V \}$ of size $V$. Let $p_k$ denote the occurrence probability of the $k$-th token. We assume $\{p_k\}_{1\leq k \leq V}$ asymptotically follows a Zipfian power-law distribution:
\begin{equation}
p_k \sim k^{-\alpha}, \quad k=1, 2, \dots, V,
\end{equation}
where $\alpha > 1$ is the decay parameter.
\end{assumption}
Assumption~\ref{Zipf_law} is grounded in quantitative linguistics. Empirical studies on large-scale corpora (e.g., Wikipedia~\cite{wikipedia}, Common Crawl~\cite{commoncrawl}) consistently demonstrate a power-law relationship between word frequency and rank. This implies that the embedding matrix constructed from these tokens will inherently exhibit a skewed, heavy-tailed singular value distribution.

\begin{assumption}
\label{quasi_orthogonal}
Let $v_k \in \mathbb{R}^d$ denote the embedding vector corresponding to the $k$-th token, normalized to unit norm such that $\|v_k\|_2 = 1$. The covariance matrix $\Sigma$ is defined as the probability-weighted outer product:
\begin{equation}
\label{eq:covariance}
\Sigma \triangleq \sum_{k=1}^{V} p_k v_k v_k^T.
\end{equation}
We posit that in the high-dimensional limit ($d \to \infty$), these embedding vectors exhibit quasi-orthogonality. Specifically, for any distinct pair of indices $i \neq j$:
\begin{equation}
\label{eq:cosine_sim}
|\langle v_i, v_j \rangle| = \mathcal{O}\left(\frac{1}{\sqrt{d}}\right).
\end{equation}
\end{assumption}

Assumption~\ref{quasi_orthogonal} is grounded in the fundamental geometry of high-dimensional spaces, where independent random vectors exhibit an tendency toward orthogonality. This phenomenon is corroborated by high-dimensional probability theory (often associated with the Johnson-Lindenstrauss Lemma~\cite{JLlemma}), which establishes that the inner product of two random unit vectors has an expectation of zero and a variance of $1/d$. Consequently, in the large $d$ regime, the cosine similarity between distinct semantic vectors scales as $\mathcal{O}(d^{-1/2})$.~\cref{eq:cosine_sim} is required for the proof of Lemma~\ref{LemmaSepcDominance}.

\subsection{Tools for Spectral Perturbation Analysis}

To quantify the spectral distortion induced by quantization noise, we leverage fundamental tools from random matrix theory~\cite{SpectralAnalysis}.

\begin{theorem}
\label{ThmStieltjesTransformContinuity}
Let $\{\mu_n\}_{n \ge 1}$ be a sequence of probability measures on $\mathbb{R}$, and let $\{m_n(z)\}_{n \ge 1}$ denote their corresponding Stieltjes transforms. If there exists a function $m(z)$ defined on the upper complex plane $\mathbb{C}^+$, such that for all $z \in \mathbb{C}^+$ the pointwise limit exists:
\begin{equation}
\lim_{n \to \infty} m_n(z) = m(z),
\end{equation}
and $m(z)$ is the Stieltjes transform of a probability measure $\mu$, then $\{\mu_n\}$ converges weakly to $\mu$:
\begin{equation}
\mu_n \xrightarrow{w.} \mu.
\end{equation}
\end{theorem}
Theorem~\ref{ThmStieltjesTransformContinuity} allows us to characterize the limiting spectral distribution  of large covariance matrices by analyzing their Stieltjes transforms, a standard technique in random matrix theory.

\begin{theorem}[\cite{Weyl}]
\label{lemma:weyl}
Let $A\in\mathbb{R}^{m\times n}$ be a matrix and $\tilde{A} = A + E$ be its perturbed counterpart, with singular values $\sigma_1 \ge \sigma_2 \ge \dots \ge 0$ and $\tilde{\sigma}_1 \ge \tilde{\sigma}_2 \ge \dots \ge 0$, respectively. Then, for every $k$, the absolute difference between the $k$-th singular values is bounded by the spectral norm of the perturbation matrix:
\begin{equation}
\label{eq:WeylIneq}
|\tilde{\sigma}_k - \sigma_k| \le \|E\|_2.
\end{equation}
\end{theorem}
Analyzing the distribution of each specific $\tilde{\sigma}_k$ directly is intractable due to complex eigenvector interactions. Theorem \ref{lemma:weyl}  allows us to decouple the individual singular value shifts from the specific noise structure, bounding them uniformly using the global spectral norm $\|E\|_2$.

\begin{theorem}[\cite{Bernstein}]
\label{thm:bernstein}
Consider a finite sequence of independent, centered random matrices $\{X_k\} \subseteq \mathbb{R}^{d_1 \times d_2}$. Assume that each matrix is bounded uniformly in spectral norm:
\begin{equation}
\|X_k\|_2 \le R \quad \text{almost surely},
\end{equation}
for some constant $R$. Let $\delta^2$ be the variance statistic defined as:
\begin{equation}
\small
\label{eq:delta}
\delta^2 = \max \left\{ \left\| \sum_{k} \mathbb{E}[X_k X_k^T] \right\|_2, \left\| \sum_{k} \mathbb{E}[X_k^T X_k] \right\|_2 \right\}.
\end{equation}
Then, for any threshold $t > 0$, the probability that the spectral norm of random matrices $Y \triangleq \sum_k X_k$ exceeds $t$ is bounded by the exponent term:
\begin{equation}
P(\|Y\|_2 \ge t) \le (d_1 + d_2) \cdot \exp \left( \frac{-t^2/2}{\sigma^2 + Rt/3} \right).
\end{equation}
\end{theorem}
Theorem~\ref{thm:bernstein} establishes that the spectral norm of the accumulated quantization noise does not grow linearly with the matrix dimensions. Instead, it concentrates sharply around zero, exhibiting sub-Gaussian tail behavior governed by the variance statistic $\delta^2$. This exponential decay of tail probabilities implies that large spectral deviations are extremely rare, thereby justifying our treatment of quantization noise as a bounded perturbation with high confidence.

\section{Low-rank Nature of Language Model Embedding}
\label{Sec:low-rank}

In this section, we provide a rigorous derivation of the spectral properties of the embedding. Our analysis proceeds in four steps: (1) connecting token frequency to the population covariance spectrum; (2) modeling the tail spectrum as effective white noise; (3) applying random matrix theory to characterize the sample covariance eigenvalues; and (4) deriving the singular value bounds for the input embeddings and  gradients. The detailed proofs are in appendix Section~\ref{sec:proof_low_rank}.

\subsection{Population Covariance and Spectral Decomposition}

We first establish the link between the linguistic distribution of tokens and the spectral decay of their embedding covariance matrix.

\begin{lemma}
\label{LemmaSepcDominance}
Under Assumption~\ref{Zipf_law} and Assumption~\ref{quasi_orthogonal}, the asymptotic behavior of the $k$-th largest eigenvalue of $\Sigma$, denoted as $\tau_k$, is governed by the probability weights $p_k$:
\begin{equation}
\tau_k \sim p_k \sim k^{-\alpha},
\end{equation}
where $\alpha>1$.
\end{lemma}
Lemma~\ref{LemmaSepcDominance} establishes that the spectral structure of $\Sigma$ inherits the power-law decay characteristic of the token probability distribution $p_k$. Based on this, we adopt the following explicit spectral model for the remainder of our analysis:
\begin{equation}
\tau_k = L \cdot k^{-\alpha}, \quad \forall k \in \{1, \dots, d\},
\end{equation}
where $L > 0$ is the energy scale and $\alpha > 1$ is the decay exponent.

To apply spectral analysis tools effectively, we must distinguish between the dominant signal components and the residual noise floor. The following lemma formalizes this decomposition.

\begin{lemma}
\label{LemmaPhaseTransition}
Let the eigenvalues of the covariance matrix $\Sigma \in \mathbb{R}^{d \times d}$ follow the power law $\tau_k = L \cdot k^{-\alpha}$ with $\alpha > 1$. Define the background noise level $\nu^2(d)$ as the mean energy of the residual eigenvalues after truncating the first $r$ principal components:
\begin{equation}
\label{eq:NoiseDef}
\nu^2(d) \triangleq \frac{1}{d-r} \sum_{j=r+1}^d \tau_j = \frac{L}{d-r} \sum_{j=r+1}^d j^{-\alpha}.
\end{equation}
Let $c$ be an arbitrary non-negative real constant. For any fixed head rank $k \in [1,r]$, the signal dominates the noise threshold:
\begin{equation}
\tau_k > \nu^2(d)\cdot (1+\sqrt{c}).
\end{equation}
Conversely, for any fixed tail rank $k \in (r,d]$, the eigenvalues fall below the threshold:
\begin{equation}
\tau_k \leq \nu^2(d)\cdot (1+\sqrt{c}).
\end{equation}
Furthermore, the noise level scales as $\nu^2(d)=\mathcal{O}(d^{-1})$.
\end{lemma}

\subsection{Matrix Phase Transitions}

A critical step in our proof is to demonstrate that the heavy-tailed residual spectrum behaves asymptotically like white noise. We achieve this by analyzing the Stieltjes transform of the spectral measure.

\begin{lemma}
\label{LemmaSTrans}
Consider the set of residual eigenvalues $\mathcal{T}_A \triangleq \{ \tau_j = L j^{-\alpha} \}_{j=r+1}^d$ defined in Lemma~\ref{LemmaPhaseTransition}. The Stieltjes transform of $\mathcal{T}_A$ is given by:
\begin{equation}
m_{Noise}(z) \triangleq \frac{1}{d-r} \sum_{j=r+1}^d \frac{1}{\tau_j - z}.
\end{equation}
Consider a reference white noise spectrum $\mathcal{T}_B$ consisting of $d-r$ identical values equal to the background noise level $\nu^2(d)$:
\begin{equation}
    \mathcal{T}_B \triangleq \underbrace{\{ \nu^2(d), \dots, \nu^2(d) \}}_{d-r},
\end{equation}
whose Stieltjes transform is:
\begin{equation}
\small
m_{White}(z) \triangleq \frac{1}{d-r} \sum_{j=r+1}^d \frac{1}{\nu^2(d) - z} = \frac{1}{\nu^2(d) - z}.
\end{equation}
Then, as $d \to \infty$, the two transforms are asymptotically equivalent:
\begin{equation}
\lim_{d \to \infty} \left| m_{Noise}(z) - m_{White}(z) \right| = 0.
\end{equation}
\end{lemma}

Combining Lemma~\ref{LemmaSTrans} with the continuity property of Stieltjes transforms (Theorem~\ref{ThmStieltjesTransformContinuity}), we derive the limiting distribution of the noise spectrum.

\begin{lemma}
\label{LemmaESD}
Let the covariance matrix $\Sigma$ be decomposed into a signal part and a noise part:
\begin{equation}
\Sigma = \Sigma_{Signal} + \Sigma_{Noise},
\end{equation}
where
\begin{equation}
\begin{split}
&\text{Spec}(\Sigma_{Signal}) = (\tau_1, \dots, \tau_r, 0, \dots, 0),\\ 
&\text{Spec}(\Sigma_{Noise})=(0, \dots, 0, \tau_{r+1}, \dots, \tau_d)
\end{split}.
\end{equation}
As $d\rightarrow \infty$, the empirical spectral distribution (ESD) of $\Sigma_{Noise}$ weakly converges to a Dirac mass centered at the noise variance $\nu^2(d)$:
\begin{equation}
\text{ESD}( \Sigma_{Noise})\xrightarrow{w.} \delta_{\nu^2(d)},
\end{equation}
where $\delta_{\nu^2(d)}$ is  defined by a probability mass of 1 located at $ \nu^2(d)$ and 0 elsewhere.
\end{lemma}

Lemma~\ref{LemmaESD} establishes a crucial fact: despite the underlying power-law decay, the tail of the covariance matrix exhibits a spectral concentration phenomenon, effectively behaving as white noise with variance $\nu^2(d)$. 

We now turn our attention to the sample covariance matrix observed during training, defined as:
\begin{equation}
S_N=\frac{1}{N}\mathbf{X}\mathbf{X}^T\in \mathbb{R}^{d\times d},
\end{equation}
where $\mathbf{X} \in \mathbb{R}^{d \times N}$ is the data matrix of token embeddings, and $N$ represents the effective batch size. Crucially, the columns of $\mathbf{X}$ are sampled independently from the vocabulary $\{v_k\}_{k=1}^V$ according to the token probabilities $\{p_k\}$. Consequently, $S_N$ serves as the empirical estimator of the population covariance $\Sigma$ defined in Assumption~\ref{quasi_orthogonal}, satisfying the unbiasedness condition:
\begin{equation}
    \mathbb{E}[S_N] = \mathbb{E}\left[\frac{1}{N}\sum_{i=1}^{N} x_i x_i^T\right] = \sum_{k=1}^{V} p_k v_k v_k^T = \Sigma,
\end{equation}
where $x_i$ is the $i$-th column of $\mathbf{X}$.  
In the asymptotic regime where $d, N \to \infty$ with the ratio $d/N \to c \in (0, \infty)$, the behavior of the sample eigenvalues $\lambda_k$ of $S_N$ is characterized by the Baik-Ben Arous-Péché (BBP) phase transition theorem~\cite{BBPTrans}.

\begin{theorem}[BBP Phase Transition]
\label{ThmBBP}
Let $\tau_k$ be the $k$-th largest eigenvalue of the population covariance $\Sigma$, and $\lambda_k$ be the corresponding eigenvalue of the sample covariance $S_N$. Given that $\text{ESD}(\Sigma_{Noise}) \xrightarrow{w} \delta_{\nu^2(d)}$, a spectral phase transition occurs at the threshold $\nu^2(d)(1+\sqrt{c})$:

\textbf{1. Super-critical regime:} If the population eigenvalue satisfies the separation condition:
    \begin{equation}
    \tau_k > \nu^2(d)(1+\sqrt{c}),
    \end{equation}
    then the corresponding sample eigenvalue $\lambda_k$ separates from the bulk spectrum and converges almost surely to a deterministic limit dependent on $\tau_k$:
    \begin{equation}
     \lambda_k \xrightarrow{a.s.} \rho(\tau_k)= \tau_k \left( 1 + \frac{c \cdot \nu^2(d)}{\tau_k - \nu^2(d)} \right).
    \end{equation}
    
\textbf{2. Sub-critical regime:} If the population eigenvalue falls below the threshold:
    \begin{equation}
    \tau_k \le \nu^2(d) (1 + \sqrt{c}),
    \end{equation}
    then $\lambda_k$ collapses to the right edge of the Marchenko-Pastur bulk distribution~\cite{MP-dist}:
    \begin{equation}
    \lambda_k \xrightarrow{a.s.} \nu^2(d) (1 + \sqrt{c})^2.
    \end{equation}

Specifically,  $\text{Spec}(\Sigma_{Signal})$ falls into the super-critical regime and $\text{Spec}(\Sigma_{Noise})$ falls into the sub-critical regime.
\end{theorem}

Theorem~\ref{ThmBBP} defines a detectability boundary at $\nu^2(d)(1+\sqrt{c})$. In the super-critical regime, signals exceeding this threshold separate from the noise bulk as outliers. In the sub-critical regime, signals collapse onto the edge of the Marchenko-Pastur bulk and become indistinguishable from noise. This justifies separating $\text{Spec}(\Sigma_{Signal})$ from $\text{Spec}(\Sigma_{Noise})$. To simplify the non-linear limit $\rho(\tau_k)$, the following lemma derives an asymptotic equivalence for $d \to \infty$, showing that dominant sample eigenvalues reconstruct the population power-law while the tail converges to a noise floor.

\begin{lemma}
\label{LemmaLambdaEqui}
The sample eigenvalues $\lambda_k$ exhibit the following asymptotic behavior:
\begin{equation}
    \lambda_k \sim 
    \begin{cases} 
        \tau_k &  1\leq k \leq r  \\
        \nu^2(d)(1+\sqrt{c})^2 & r< k \leq d 
    \end{cases},
\end{equation}
where $\tau_k \in \operatorname{Spec}(\Sigma_{\text{Signal}})$ when $1\leq k \leq r$, and $\tau_k \in \operatorname{Spec}(\Sigma_{\text{Noise}})$  when $r< k \leq d$.
\end{lemma}

\subsection{Singular Values Bounds}

Finally, we translate these results from the covariance domain back to the singular values of the data embedding matrix $\mathbf{X}$ and the weight gradient $\nabla_\mathbf{W} = \mathbf{X}^T \mathbf{G}$, where $\mathbf{G}$ is the gradient of the loss with respect to the output. Since $S_N = \frac{1}{N}\mathbf{X}\mathbf{X}^T$, the singular values of $\mathbf{X}$ are related to the eigenvalues of $S_N$ by $\sigma_k(\mathbf{X}) = \sqrt{N \lambda_k}$.

\begin{theorem}
\label{ThmX}
Let $\sigma_k(\mathbf{X})$ be the $k$-th largest singular value of the data embedding matrix $\mathbf{X} \in \mathbb{R}^{d \times N}$. As $d, N \to \infty$ with $d/N \to c$, the singular value spectrum follows a  power-law profile:
\begin{equation}
\sigma_k(\mathbf{X}) \sim 
\begin{cases}
\sqrt{NL}\cdot k^{-\alpha/2}, & 1\leq k\leq r \\
\sqrt{N\nu^2(d)}\cdot(1+\sqrt{c}), & r<k\leq d 
\end{cases}.
\end{equation}
Furthermore, there exists a constant $C_1$ such that:
\begin{equation}
\sigma_k(\mathbf{X}) \leq 
\begin{cases}
 C_1 \cdot k^{-\alpha/2}, & 1\leq k\leq r\\
C_1\cdot r^{-\alpha/2}, & r<k\leq d
\end{cases}.
\end{equation}
\end{theorem}

Applying the chain rule of calculus and standard spectral norm inequalities, we derive the spectral bound for the weight gradients.

\begin{theorem}
\label{ThmG}
Consider a linear layer $\mathbf{Y} = \mathbf{XW}$, where $\mathbf{Y}$ is the output matrix, $\mathbf{X}$ is the data embedding matrix and $\mathbf{W}$ is the weight matrix. Let $\mathbf{G}$ be the gradient of  loss with respect to the output. Assuming training stability (i.e., $\|\mathbf{G}\|_2 \leq M$ for a finite constant $M$), the singular values of the weight gradient matrix $\nabla_\mathbf{W} = \mathbf{X}^T \mathbf{G}$ are bounded by:
\begin{equation}
\sigma_k(\nabla_\mathbf{W}) \leq 
\begin{cases}
(MC_1) \cdot k^{-\alpha/2}, & 1\leq k\leq r\\
(MC_1)\cdot r^{-\alpha/2}, & r<k\leq d
\end{cases}.
\end{equation}
\end{theorem}

\section{Quantization Increases Stable Rank}
\label{sec:quantization}

 In Theorem~\ref{ThmX} and Theorem~\ref{ThmG}, we established that key components of the training process, $\mathbf{X}$ and  $\nabla_\mathbf{W}$,  exhibit a power-law spectral decay. To provide a unified analysis, we abstract these specific matrices into a generic stochastic matrix $A$ in this section. Therefore, the matrix $A$ discussed below should be interpreted as a theoretical surrogate for either the high-precision embedding $\mathbf{X}$ or the gradient $\nabla_\mathbf{W}$, both of which satisfy the power-law assumption. The detailed proofs are in appendix Section~\ref{sec:quant_increase_sr}.

\subsection{Problem Setup and Spectral Assumptions}

Let $A \in \mathbb{R}^{m \times n}$ denote the target high-precision matrix (representing $\mathbf{X}$ or $\nabla_\mathbf{W}$). Consistent with the empirical observations in LLMs and our theoretical derivations in Section~\ref{Sec:low-rank}, we formalize the spectral structure of $A$ with the following assumption.
\begin{assumption}
\label{ass:power_law}
The singular values $\{\sigma_k(A)\}$ of the matrix $A$ are assumed to follow the profile:
\begin{equation}
    \sigma_k(A) \sim \begin{cases}
    \mu \cdot k^{-\alpha/2}, & 1 \le k \le r \\
    \mu \cdot r^{-\alpha/2}, & r < k \le \min(m,n)
    \end{cases}
    \label{power_law},
\end{equation}
where $\mu > 0$ is the spectral magnitude constant, and $\alpha > 1$ is the decay rate parameter characteristic of the model's domain.
\end{assumption}
Let $\tilde{A}$ denote the component-wise quantized counterpart of $A$. The quantization error matrix is defined as the difference $E = \tilde{A} - A$, with entries $E_{ij}=e_{ij}$. To facilitate a component-wise spectral analysis, we decompose the error matrix $E$ into a sum of independent basis matrices:
\begin{equation}
    E = \sum_{i=1}^{m}\sum_{j=1}^{n} Z_{ij},
\end{equation}
where each $Z_{ij} \in \mathbb{R}^{m \times n}$ is a sparse matrix containing the error $e_{ij}$ at entry $(i,j)$ and zeros elsewhere.

In the context of modern low-precision training mechanisms, such as MXFP4 and NVFP4, the underlying quantization operators can be unified under a block-wise symmetric uniform quantization formulation. Specifically, for an input scalar $a_{ij}=A_{ij}$, the quantizer $Q_s(a_{ij})$ is defined as:
\begin{equation}
\label{quantizer}
    Q_s(a_{ij})=s \cdot \text{round}\left(\frac{a_{ij}}{s}\right),
\end{equation}
where $s=a_{\max}/L$ represents the quantization step size. Here, $a_{\max}=\max_{i,j}{|A_{ij}|} $ represents the dynamic range, and $L$ is the maximum representable integer value for the given bit-width (e.g., $L=7$ for signed 4-bit integer representation). We adopt the standard statistical assumption regarding the distribution of the quantization error $e_{ij}=Q_s(a_{ij})-a_{ij}$, as stated in Assumption~\ref{uniform_ass}.
% \begin{assumption}[\cite{Quantization_noise}]
% \label{uniform_ass}
%     For the matrix $A$, the quantization error $e_{ij}=Q_s(a_{ij})-a_{ij}$ is assumed to be uniformly distributed over the interval $[-s/2, s/2]$.
% \end{assumption}
\begin{assumption}
\label{uniform_ass}
The quantization error $e_{ij}=Q_s(a_{ij})-a_{ij}$ is assumed to be a centered random variable with zero mean $\mathbb{E}[e_{ij}] = 0$  and finite variance $\text{Var}(e_{ij})= B\propto s^2$.
\end{assumption}

% Under Assumption~\ref{uniform_ass}, the quantization error is bounded by $|e_{ij}|\leq s/2$, and its first and second moments are given by:
% \begin{equation}
%     \mathbb{E}[e_{ij}] = 0, \quad \text{Var}(e_{ij}) = \mathbb{E}[e_{ij}^2] = \frac{s^2}{12}.
% \end{equation}

\subsection{Bounding the Spectral Norm of Noise}

A primary challenge in analyzing quantization effects is the stochastic nature of the error. To derive a deterministic control over the error's spectral norm $\|E\|_2$, we employ Theorem~\ref{thm:bernstein}. To apply this inequality to our specific error decomposition $E = \sum_{i,j} Z_{ij}$, we must first compute the variance statistic $\delta^2$ (\cref{eq:delta}), which requires evaluating the spectral norms of the expected covariance sums. The following lemma simplifies this calculation by establishing the independence of cross-terms.

\begin{lemma}
\label{LemmaCrossTerms}
For the quantization error matrix $E$ decomposed into independent basis matrices $Z_{ij}$ with $e_{ij}$ following Assumption~\ref{uniform_ass}, the expected second-moment matrix sums satisfy:
\begin{equation}
\begin{aligned}
    &\sum_{i,j} \mathbb{E}[Z_{ij} Z_{ij}^T] = \mathbb{E}[EE^T], \\
    &\sum_{i,j} \mathbb{E}[Z_{ij}^T Z_{ij}] = \mathbb{E}[E^T E].
\end{aligned}
\end{equation}
\end{lemma}

Lemma~\ref{LemmaCrossTerms} is critical as it decouples the spatial correlations within the error matrix. It essentially states that, due to independence, the variance of the sum equals the sum of the variances, even in the matrix domain. By Theorem~\ref{thm:bernstein} and Lemma~\ref{LemmaCrossTerms}, the following theorem holds.
\begin{theorem}
\label{ThmBernstein2}
Consider a series of independent basis matrices $\{Z_{ij}\}$ of dimension $m \times n$ where the error terms $e_{ij}$ follow Assumption~\ref{uniform_ass}. Let $\delta^2$ be the variance statistic defined as:
\begin{equation}
\small
    \delta^2 = \max \left\{ \left\| \sum_{i,j} \mathbb{E}[Z_{ij} Z_{ij}^T] \right\|_2, \left\| \sum_{i,j} \mathbb{E}[Z_{ij}^T Z_{ij}] \right\|_2 \right\}.
\end{equation}
Then, for any threshold $t > 0$, the probability that the spectral norm of the sum $E = \sum_{i,j} Z_{ij}$ exceeds $t$ is bounded by:
\begin{equation}
\small
\label{eq:tail_bound}
     P(\|E\|_2 \ge t) \le (m+n) \cdot \exp \left( -\frac{t^2}{C \cdot \max(m,n) \cdot s^2} \right),
\end{equation}
where $C$ is a universal constant related to the sub-Gaussian norm of the noise.
\end{theorem}
The derivation above yields a Gaussian-like tail bound, quantitatively demonstrating that the global spectral noise $\|E\|_2$ scales approximately with $\sqrt{\max(m,n)} \cdot s$. While this suggests that the absolute spectral error increases for extremely wide layers, the stability of the representation is dictated by the relative impact of this noise, which we analyze next.

Let us define the tail bound function $F(t)$ as the right-hand side of Equation~\eqref{eq:tail_bound}:
\begin{equation}
\label{F(t)}
    F(t) \triangleq (m+n) \cdot \exp \left( -\frac{t^2}{C \cdot \max(m,n) \cdot s^2} \right).
\end{equation}
Note that the derivative $F'(t) < 0$ for all $t > 0$, ensuring that $F(t)$ is strictly monotonically decreasing. This monotonicity is crucial for the existence of its inverse function in the subsequent analysis.

\subsection{Relative Error Analysis of Singular Values}

Having established a probabilistic upper bound for the global spectral error $\|E\|_2$, we now investigate how this global noise translates into perturbations of individual singular values. To bridge the gap between the spectral norm of the error matrix and the specific deviations of singular values, we define the relative error violation event $\mathcal{R}(\epsilon_k)$ over the sample space $\Omega$. Specifically, $\mathcal{R}(\epsilon_k)$ collects all outcomes $\omega \in \Omega$ where the relative perturbation exceeds $\epsilon_k$:
\begin{equation}
    \mathcal{R}(\epsilon_k) \triangleq \left\{ \omega \in \Omega : \frac{|\sigma_k(\tilde{A}(\omega)) - \sigma_k(A)|}{\sigma_k(A)} > \epsilon_k \right\}.
\end{equation}
For brevity, we suppress the dependence on $\omega$ in the notation and refer to the event as $\mathcal{R}(\epsilon_k) = \left\{ \frac{|\sigma_k(\tilde{A}) - \sigma_k(A)|}{\sigma_k(A)} > \epsilon_k \right\}$.

\begin{theorem}
\label{thm:relative_error}
For a fixed confidence level $1-\theta$ where $\theta \in (0,1)$, the maximum relative error bound $\epsilon_k(\theta)$ induced by quantization is strictly inversely proportional to the magnitude of the singular value $\sigma_k(A)$:
\begin{equation}
\label{eq:inverse}
    \epsilon_k(\theta) \propto \frac{1}{\sigma_k(A)}.
\end{equation}
Furthermore, for any two singular values with magnitudes $\sigma_i(A) > \sigma_j(A) > 0$, under the same confidence level $1-\theta$, it strictly holds that:
\begin{equation}
    \epsilon_i(\theta) < \epsilon_j(\theta).
\end{equation}
\end{theorem}

Theorem \ref{thm:relative_error} establishes the theoretical core of our analysis, revealing a fundamental spectral asymmetry: while quantization adds a roughly uniform noise floor to the entire spectrum, the relative damage is highly non-uniform. Dominant singular values are sufficiently large to suppress this noise, resulting in negligible relative perturbation. Conversely, tail singular values operate in a low signal-to-noise regime, making them highly susceptible to distortion.

\subsection{ Failure Probability under Power-Law Decay}

In this subsection, we concretize the relative error analysis by calculating the explicit failure probability.  Define $\theta_k$ as the failure probability for the $k$-th singular value. The following theorem gives an upper bound of $\theta_k$

\begin{theorem}
\label{thm:asymptotic_robustness_exact}
Consider a stochastic matrix $A\in \mathbb{R}^{m\times n}$ satisfying Assumption~\ref{ass:power_law} and Assumption~\ref{uniform_ass}. Let $R_k \triangleq \frac{|\sigma_k(\tilde{A}) - \sigma_k(A)|}{\sigma_k(A)}$ be the relative error.
For any fixed relative error tolerance $\eta > 0$, define $\theta_k$ to be the probability that  $R_k>\eta$. Then $\theta_k$ is bounded by:
\begin{equation}
    \theta_k \lesssim (m+n) \cdot \exp\left( - \frac{\eta^2 \mu^2}{CDs^2} \cdot \xi_k \right),
    \label{theta_bound}
\end{equation}
where
\begin{equation}
    \xi_k = \begin{cases} k^{-\alpha}, & 1 \le k \le r \\ r^{-\alpha}, & r < k \le \min(m,n) \end{cases},
\end{equation}
$D\triangleq\max(m,n)$ and  $C$ is a universal constant related to the sub-Gaussian norm of the noise.
\end{theorem}

Theorem \ref{thm:asymptotic_robustness_exact} establishes the failure probability profile under the asymptotic spectral assumption. Recall that the quantization step size is defined as $s = a_{\max}/L$.  Based on Assumption~\ref{ass:power_law}, we have $\sigma_1(A) \sim \mu$. Consequently, the step size satisfies the scaling relation $s \sim \mu/L$. Substituting this asymptotic relation into~\cref{theta_bound}, the failure probability bound takes the following form:
\begin{equation}
    \theta_k \lesssim (m+n) \cdot \exp\left( - \frac{\eta^2 L^2}{C \cdot D} \cdot \xi_k \right).
    \label{eq:simplified_resolution_bound}
\end{equation}
Consider the principal components where $1\leq k\leq r$. In this regime $\xi_k=k^{-\alpha}$. As the quantization resolution $L$ increases, the exponent diverges to negative infinity:
    \begin{equation}
        \lim_{L \to \infty} \left( - \frac{\eta^2 L^2}{C \cdot D} \cdot k^{-\alpha} \right) = -\infty.
    \end{equation}
    Consequently, the failure probability vanishes:
    \begin{equation}
        \lim_{L \to \infty} \theta_k = 0,
    \end{equation}
which proves that the dominant singular values are robust against quantization noise, provided the bit-width is sufficiently large.  Consider the residual components where $r<k\leq \min(m,n)$. In this regime $\xi_k=r^{-\alpha}$. When $r\rightarrow \infty$, the limit holds:
\begin{equation}
    \lim_{r\rightarrow\infty} (m+n) \cdot \exp\left( - \frac{\eta^2 L^2}{C \cdot D} \cdot r^{-\alpha} \right)=(m+n).
\end{equation}
The bound relaxes to the trivial limit $(m+n)$, implying that the relative error is no longer suppressed. This mathematically formalizes the spectral flattening: in the tail singular values, the quantization noise magnitude is comparable to the signal floor, rendering the spectral structure indistinguishable from noise.

\subsection{Quantization Increases Stable Rank}

Having established the non-uniform impact of quantization noise under \cref{ass:power_law}, we now apply these findings to explain the behavior of the matrix stable rank.

\begin{theorem}
\label{thm:stable_rank_impact}
Let $A \in \mathbb{R}^{m \times n}$ be a matrix satisfying Assumption~\ref{ass:power_law}, and let $\tilde{A}$ be its quantized counterpart. Then the stable rank of the matrix $A$ strictly increases after quantization:
\begin{equation}
    S_r(\tilde{A}) > S_r(A).
\end{equation}
\end{theorem}

\section{Experiments}
\label{sec:experiments}
\subsection{Model Architecture and Dataset}

To verify the analysis, we employ a decoder-only Transformer architecture following the configuration of GPT-2-124M~\cite{gpt-2}. Specifically, the model comprises $L=12$ Transformer layers with a hidden dimension of $d_{\text{model}}=768$ and $H=12$ attention heads. The context window size is set to 1024 tokens. We train all models on the OpenWebText dataset~\cite{openwebtext}, a large-scale collection of web text widely utilized for language model pre-training. All experiments are implemented using the PyTorch framework and executed on a computational node equipped with 8 GPUs. We utilize a global batch size of 1280, which is achieved using a local batch size of 4 per GPU combined with 40 steps of gradient accumulation. The models are trained with AdamW optimizer~\cite{adamw} with a weight decay coefficient of $\beta=0.1$. The learning rate is initialized at $1.5 \times 10^{-4}$ and features a linear warm-up over the first 20 steps. Furthermore, to mitigate the risk of gradient explosion, we apply gradient clipping with a maximum norm of $1.0$.  All FP4 training in this work
adopts W4A4G4 quantization, where weights,
activations, and gradients are represented in
the E2M1 NVFP4 format. Due to Nvidia’s
closed-source FP4 training software stack, native hardware-supported FP4 training is not currently accessible; consequently, our experiments
with NVFP4 are conducted through simulation
in \texttt{BF16}. More experiments are in appendix Section~\ref{sec:more_exp}.

\subsection{Empirical Verification of Spectral Flattening}
\label{subsec:spectral_analysis}

We empirically validate the predicted spectral degradation through a layer-wise analysis of the Multilayer Perceptron (MLP) blocks in GPT-2-124M. Specifically, we evaluate the singular value spectra of the \texttt{BF16} baseline against its  NVFP4 quantized counterpart.

The left panel of Figure~\ref{fig:singular_values} elucidates the mechanism via the component-wise distortion ratio $\epsilon_k$. We observe disproportionate attenuation in dominant singular values (top 30\%, ratio $\approx 0.02$--$0.05$), which directly inflates the stable rank by suppressing the spectral norm. Conversely, the tail components exhibit relative preservation ( $\epsilon_k\to 0.18$). This asymmetry corroborates Theorem~\ref{thm:relative_error}, confirming that the quantization noise floor effectively obstructs the decay of head spectra. The right panel of Figure~\ref{fig:singular_values} confirms that pre-quantization embeddings exhibit a log-linear spectral decay, validating the power-law  result in Theorem~\ref{ThmX}. Following NVFP4 quantization, we detect a deterministic increase in stable rank ($2.52 \to 2.71$), consistent with the spectral flattening predicted by Theorem~\ref{thm:stable_rank_impact}.

\begin{figure}[htbp]
    \centering
    \includegraphics[width=\linewidth]{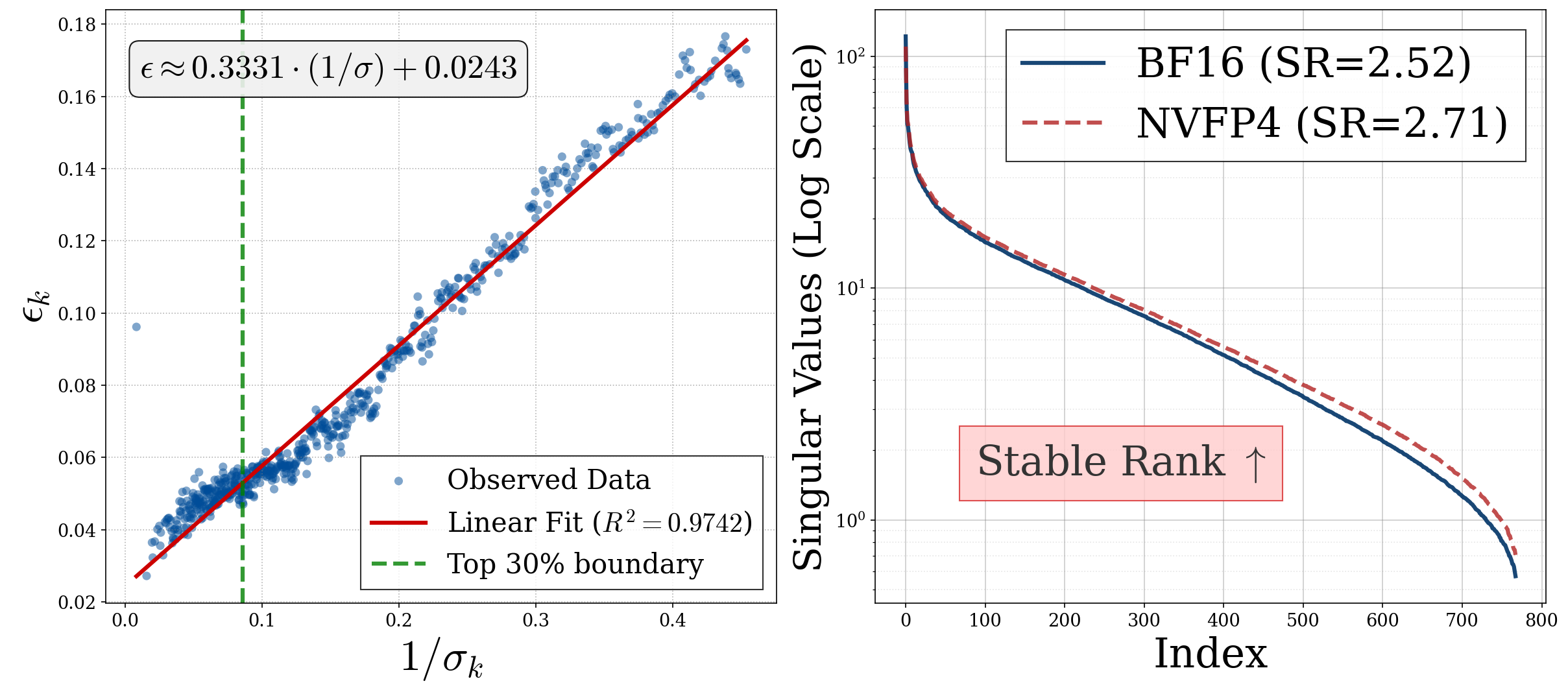}
    \caption{Spectral analysis of MLP embedding matrices. (Left) The ratio of quantized to original singular values across indices. (Right) Logarithmic distribution of singular values with \texttt{BF16} and \texttt{NVFP4} quantization. }
    \label{fig:singular_values}
\end{figure}

\subsection{Spectral Evolution and Representational Collapse}
\label{subsec:dynamic_spectral}

Extending the analysis of Theorem~\ref{thm:stable_rank_impact}, we investigate the temporal evolution of spectral properties. Figure~\ref{fig:spectral_evolution} contrasts the singular value spectra of \texttt{BF16} and \texttt{NVFP4} embedding matrices in MLP layers across representative training steps.

The empirical evidence reveals a bifurcation in spectral dynamics, characterizing two distinct phases of low-precision training. In the initial iterations (Steps 0--10), observations corroborate our theoretical predictions. As shown in the top row of Figure~\ref{fig:spectral_evolution}, the quantized embeddings (\texttt{NVFP4}) display an elevated tail spectrum relative to \texttt{BF16} baseline. Consequently, the stable rank of the quantized representations increases. This confirms that quantization introduces a noise floor that obscures fine-grained spectral structures, resulting in a generalized flattening of the distribution.

However, as training progresses (Steps 50--100), a pathological divergence emerges. The bottom row of Figure~\ref{fig:spectral_evolution} demonstrates that the stable rank of quantized embedding declines rapidly, deviating from the monotonic increase observed in the \texttt{BF16} baseline. By Step 100, the \texttt{BF16} model attains a high stable rank (SR=1.08), indicative of effective dimensionality expansion and feature encoding. Conversely, the NVFP4 embedding degenerates into a low-rank state (SR=1.00). This crossover suggests that cumulative gradient errors shift the regime from static flattening to dynamic representational collapse. This process compresses the model capacity into a few dominant singular values, effectively eliminating the heavy-tailed spectral structure essential for semantic encoding.

\begin{figure}[htbp]
    \centering
    \includegraphics[width=\linewidth]{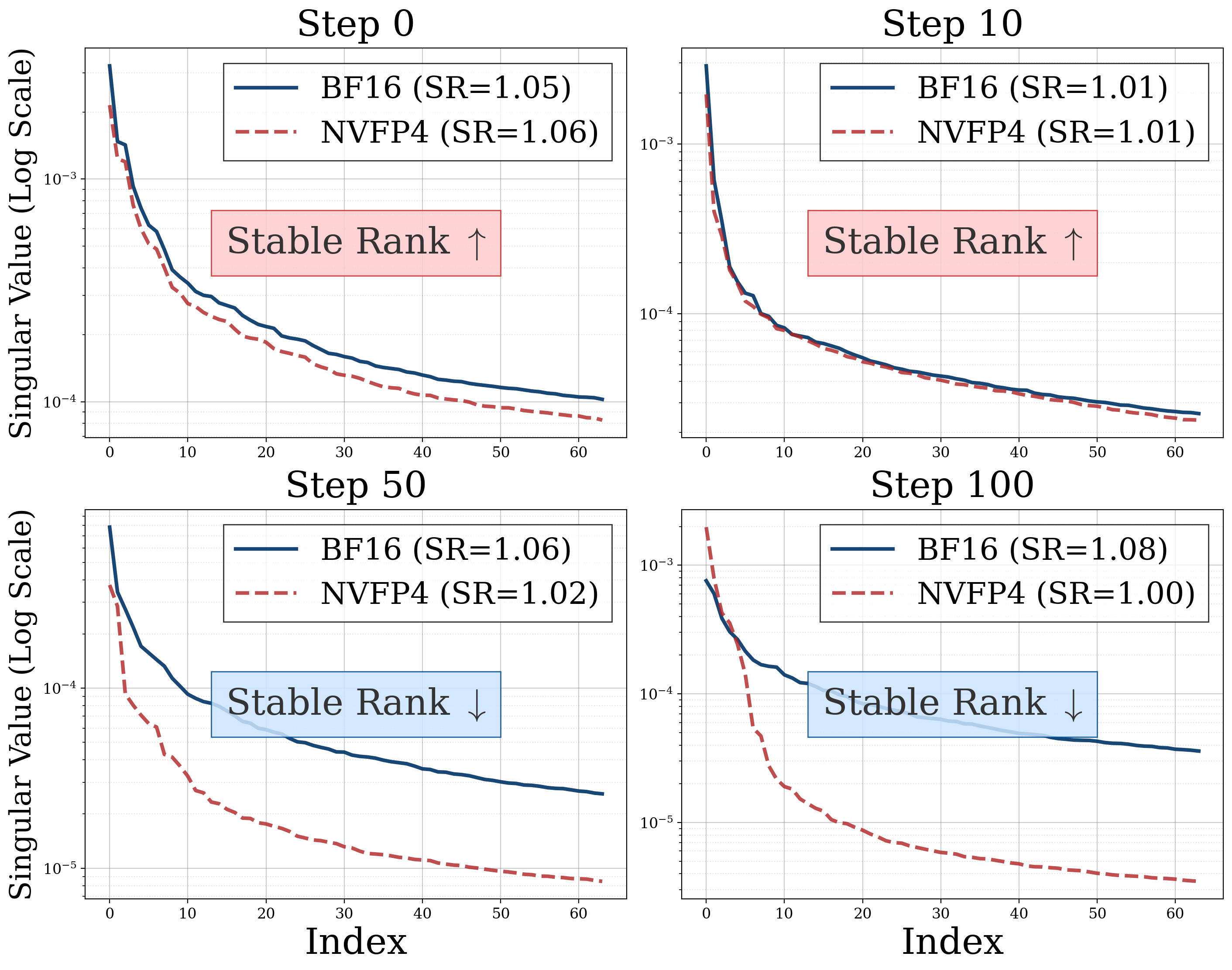}
    \caption{Spectral evolution of training. Top: Initial noise inflates Stable Rank (SR). Bottom: Later steps show representational collapse.}
    \label{fig:spectral_evolution}
\end{figure}
\section{Conclusion and Future Work}
\label{sec:conclusion}
This work unifies Zipfian statistics and Random Matrix Theory to demonstrate that uniform quantization disrupts the heavy-tailed spectral structure of LLMs, causing spectral flattening and a strict increase in stable rank. Our empirical analysis shows that this distortion precipitates a phase transition into representational collapse. This finding challenges unstructured perturbation models by characterizing quantization as a geometric degradation of the optimization landscape rather than simple additive noise.

Future research should focus on designing non-uniform quantization schemes grounded in power-law statistics to preserve tail fidelity, alongside employing spectral regularization to prevent rank collapse. Furthermore, extending this spectral analysis to non-linearities such as softmax and attention is essential to fully characterize quantization effects. These directions collectively establish the theoretical foundation for robust low-precision training.

\bibliographystyle{unsrt}

\bibliography{references}

\newpage
\appendix

\section{More Related Works}
\label{sec:related_work}

The challenge of ultra-low-bit quantization in Large Language Models (LLMs) is a problem of preserving spectral structure under extreme compression. In this section, we review the theoretical foundations of this challenge, tracing the trajectory from the statistical physics of natural language to the spectral properties of neural networks, and subsequently analyzing the evolution of quantization algorithms designed to mitigate spectral degradation.

\subsection{Statistical and Spectral Foundations}
The spectral attributes of LLMs are not mere artifacts of random initialization; rather, they are intrinsically shaped by the statistical distribution of training data. As established by \cite{Zipf}, word frequency in natural language adheres to a strict power-law distribution $f \propto r^{-\alpha}$. This distribution has been extensively validated across web-scale corpora, including Common Crawl \cite{commoncrawl} and Wikipedia \cite{wikipedia}. Recent theoretical analysis by \cite{debowski2025zipf} suggests that observed Neural Scaling Laws are deductive consequences of these Zipfian statistics.

In the context of neural representations, this heavy-tailed data distribution induces geometric anisotropy in the embedding space. \cite{ethayarajh2019contextual} demonstrated that contextual representations occupy a \textit{narrow cone} topology, while \cite{gao2019representation} formalized this phenomenon as representation degeneration. Although early frameworks postulated that this generative process governs geometric structure \cite{arora2016}, recent analysis by \cite{ren2025variability} empirically confirms that high-frequency tokens exhibit distinct variability patterns compared to rare tokens, establishing that token frequency dictates the local curvature of the embedding manifold. This phenomenon is often exacerbated by the frequency bias inherent in the softmax layer \cite{jiang2019}.

Chronologically, strategies to mitigate this collapse have evolved from heuristic subtraction \cite{All-but-the-Top} to density-aware calibration, such as normalizing flows \cite{li2020} and standard whitening \cite{su2021}. Most recently, \cite{ref14} introduced Zipfian whitening, advancing this methodology by explicitly leveraging word frequency statistics to flatten the spectrum. While \cite{aghajanyan2021intrinsic} leveraged the resulting low intrinsic dimensionality to show models can be fine-tuned in low-rank subspaces, \cite{jha2025spectral} argues that the fragile long tail—often dismissed as noise—encodes critical semantic distinctions via a tail-first growth mechanism. This emphasis on the heavy-tailed spectrum aligns with rigorous characterizations from Random Matrix Theory (RMT). \cite{martin2021} identified heavy-tailed self-regularization (HT-SR) as a hallmark of high-performance deep networks. Parallel research on the Hessian spectrum \cite{sagun2017empirical, ghorbani2019investigation} confirms that the loss landscape consists of a massive bulk and a few extreme outliers, often associated with flat minima \cite{hochreiter1997flat}. Early work on Optimal Brain Damage (OBD) by \cite{lecun1990optimal} utilized these properties for pruning. Subsequently, \cite{dong2019hawq} leveraged this in Hessian-Aware Quantization, demonstrating that layers with larger spectral norms are disproportionately sensitive to noise. More recently, the Muon optimizer \cite{liu2025muon} explicitly manages these spectral properties during training, reinforcing the necessity of maintaining specific spectral profiles for scalable learning. Table~\ref{tab:foundations_literature} summarized the literature on statistical and spectral foundations.

\begin{table*}[htbp]
\centering
\caption{Summary of Literature on Statistical and Spectral Foundations.}
\label{tab:foundations_literature}
\resizebox{\textwidth}{!}{%
\begin{tabular}{@{}lp{8cm}l@{}}
\toprule
\textbf{Category} & \textbf{Key Contributions \& Findings} & \textbf{Citations} \\ \midrule
\textbf{Zipfian Statistics} & Establishes power-law distribution ($f \propto r^{-\alpha}$) in natural language and validates it on web-scale corpora; links Zipf's law to Neural Scaling Laws. & \cite{Zipf, commoncrawl, wikipedia, debowski2025zipf} \\ \midrule
\textbf{Anisotropy \& Geometry} & Identifies representation degeneration and the narrow cone effect; links token frequency to local manifold curvature and softmax bias. & \cite{ethayarajh2019contextual, gao2019representation, ref13, arora2016, ren2025variability, jiang2019} \\ \midrule
\textbf{Isotropy Restoration} & Evolution from heuristic subtraction (All-but-the-Top) to Normalizing Flows. & \cite{All-but-the-Top, li2020, su2021, ref14} \\ \midrule
\textbf{Dimensionality \& Tails} & Debates on intrinsic dimensionality vs. the importance of the fragile long tail for semantic distinctions and capacity growth. & \cite{aghajanyan2021intrinsic, jha2025spectral} \\ \midrule
\textbf{Spectral Theory (RMT)} & Characterizes heavy-tailed self-regularization (HT-SR) in trained weights; analyzes empirical spectral densities. & \cite{martin2021} \\ \midrule
\textbf{Hessian \& Optimization} & Analyzes Hessian spectrum (bulk vs. outliers, flat minima); Optimal Brain Damage (OBD); Hessian-Aware Quantization (HAWQ); Spectral control via Muon optimizer. & \cite{sagun2017empirical, ghorbani2019investigation, hochreiter1997flat, lecun1990optimal, dong2019hawq, liu2025muon} \\ \bottomrule
\end{tabular}%
}
\end{table*}

\subsection{Quantization Pathology and Algorithmic Evolution}
These theoretical properties manifest practically as the outlier phenomenon \cite{bondarenko2021understanding}, which constitutes the primary barrier to low-bit quantization. \cite{dettmers2022llm} observed that when parameter counts exceed 6.7B, systematic massive activations emerge \cite{sun2024massive}. According to classic quantization theory \cite{Quantization_noise}, these outliers compel uniform quantizers to adopt excessive step sizes, causing the fine-grained information in the distribution bulk to collapse. \cite{wei2022outlier} further identified LayerNorm as a key factor amplifying these variances.

To address the spectral challenges posed by outliers, quantization strategies have evolved from simple decomposition to complex geometric transformations. Early mitigation strategies, such as LLM.int8() \cite{dettmers2022llm}, employed mixed-precision decomposition to isolate outliers. Subsequent approaches utilizing mathematical equivalences, such as SmoothQuant \cite{xiao2023smoothquant} and AWQ \cite{lin2024awq}, sought to migrate quantization difficulty from activations to weights.

Recent research has shifted towards geometric transformations to fundamentally alter the spectral distribution. Methods like QuIP \cite{chee2023quip} and QuIP\# \cite{tseng2024quip} introduce Randomized Hadamard Transforms to spectrally whiten the distribution, approximating a Gaussian distribution by leveraging the Johnson-Lindenstrauss lemma \cite{johnson1984extensions}. \cite{ashkboos2024quarot} extended this approach to end-to-end inference in QuaRot. Most recently, \cite{xie2025mhc} proposed limiting signal propagation spectral norms via manifold constraints.

Concurrently, hardware-aware formats have evolved to support these algorithmic advances. While initial efforts explored \texttt{FP8} formats \cite{fp8}, newer standards like \texttt{NVFP4} \cite{nvidia2025blackwell} and \texttt{MXFP4} \cite{mxfp4} adopt block-wise scaling to perform local spectral normalization, thereby preserving representational capacity in low-precision regimes. Despite these engineering advances, the specific theoretical mechanism linking the breakdown of reasoning capabilities in quantized models to the fragility of the Zipfian spectral tail remains an open question. Table~\ref{tab:quantization_literature} summarized the literature on quantization pathology and algorithmic evolution.

\begin{table*}[htbp]
\centering
\caption{Summary of Literature on Quantization Pathology and Algorithmic Evolution. }
\label{tab:quantization_literature}
\resizebox{\textwidth}{!}{%
\begin{tabular}{@{}lp{8cm}l@{}}
\toprule
\textbf{Category} & \textbf{Key Contributions \& Mechanisms} & \textbf{Citations} \\ \midrule
\textbf{Outlier Pathology} & Identification of the outlier phenomenon and massive activations in LLMs (>6.7B); analyzes the role of LayerNorm in variance amplification. & \cite{bondarenko2021understanding, dettmers2022llm, sun2024massive, wei2022outlier} \\ \midrule
\textbf{Classic Theory} & Foundational quantization noise theories assuming uniform distribution (often violated by LLM outliers). & \cite{Quantization_noise} \\ \midrule
\textbf{Decomp. \& Migration} & \textbf{Decomposition}: Mixed-precision (INT8/FP16) isolation. \newline \textbf{Migration}: Mathematically shifting quantization difficulty from activations to weights (SmoothQuant, AWQ). & \cite{dettmers2022llm, xiao2023smoothquant, lin2024awq} \\ \midrule
\textbf{Rotation \& Whitening} & Uses Randomized Hadamard Transforms (RHT) and JL Lemma to whiten spectral distribution and delocalize outliers (QuIP, QuIP\#, QuaRot). & \cite{chee2023quip, tseng2024quip, johnson1984extensions, ashkboos2024quarot} \\ \midrule
\textbf{Manifold Constraints} & Explicitly limits signal propagation spectral norms via manifold projections (mHC). & \cite{xie2025mhc} \\ \midrule
\textbf{Hardware Formats} & Evolution from \texttt{FP8} to block-wise scaling formats (\texttt{NVFP4}, \texttt{MXFP4}) designed for local spectral normalization. & \cite{fp8, nvidia2025blackwell, mxfp4} \\ \bottomrule
\end{tabular}%
}
\end{table*}

\section{Supplementary Definitions}
\label{sec:supp_def}
To ensure the self-contained nature of this work, we provide formal definitions of the mathematical concepts and theoretical frameworks utilized in our spectral analysis. 
\begin{definition}[Stable Rank]
\label{def:stable_rank}
For a stochastic matrix $A \in \mathbb{R}^{m \times n}$, the \textit{stable rank}, denoted by $S_r(A)$, is defined as the ratio of its squared Frobenius norm to its squared spectral norm:
\begin{equation}
    S_r(A) \triangleq \frac{\|A\|_F^2}{\|A\|_2^2},
\end{equation}
where $\|A\|_F = \sqrt{\sum_{i,j} A_{ij}^2}$ is the Frobenius norm and $\|A\|_2 = \sigma_{\max}(A)$ is the spectral norm.
\end{definition}
\begin{definition}[Stieltjes Transform]
The Stieltjes transform of a probability measure $\mu$ is defined for $z \in \mathbb{C} \setminus \text{supp}(\mu)$ as:
\begin{equation}
m_{\mu}(z) = \int_{\mathbb{R}} \frac{1}{\lambda - z} d\mu(\lambda).
\end{equation}
For a matrix with $n$ discrete eigenvalues $\{\lambda_i\}_{i=1}^n$, the transform takes the empirical form:
\begin{equation}
m_n(z) = \frac{1}{n} \sum_{i=1}^n \frac{1}{\lambda_i - z}.
\end{equation}
\end{definition}
Stieltjes transform is used to prove the spectral convergence of noise models in Lemma~\ref{LemmaSTrans} and Lemma~\ref{LemmaESD}.

\begin{definition}[Empirical Spectral Distribution, ESD]
For a symmetric matrix $M \in \mathbb{R}^{d \times d}$ with eigenvalues $\lambda_1, \dots, \lambda_d$, the ESD is defined as:
\begin{equation}
\mu_M = \frac{1}{d} \sum_{i=1}^d \delta_{\lambda_i}
\end{equation}
where $\delta$ is the Dirac delta measure. The ESD converges to a deterministic measure as $d \to \infty$.
\end{definition}
The ESD is employed to characterize the global density of eigenvalues, serving as the basis for proving that heavy-tailed quantization noise converges to a deterministic noise bulk in Lemma~\ref{LemmaESD} and Theorem~\ref{ThmBBP}.

\section{ Math Proof of Low-rank Nature of Language Model Embedding}
\label{sec:proof_low_rank}

\begin{proof}[Proof of Lemma~\ref{LemmaSepcDominance}]
    We write~\cref{eq:covariance} in matrix form. Let $U = [v_1, v_2, \dots, v_V] \in \mathbb{R}^{d \times V}$ be the matrix of eigenvectors, $\Lambda_p = \text{diag}(p_1, \dots, p_V)$ be the probability diagonal matrix. Then, 

    \begin{equation}
\Sigma = U \Lambda_p U^T.
\end{equation}
By Assumption~\ref{Zipf_law}, the probability weights satisfy $p_1 > p_2 > \dots > p_V > 0$, and $p_k \sim k^{-\alpha}$, which means that $p_k$ has a large spectral gap. According to the singular values decomposition (SVD) duality property in linear algebra, $\Sigma$'s non-zero eigenvalues are exactly the same as those of the matrix $M = \Lambda_p^{1/2} U^T U \Lambda_p^{1/2}$. We analyze the matrix $M\in \mathbb{R}^{V\times V}$. Let $G = U^T U$ be the Gram matrix~\cite{Weyl} of the eigenvectors. Due to vector normalization, 
\begin{equation}
(G)_{ii} = \|v_i\|^2 = 1.
\end{equation}
By Assumption~\ref{quasi_orthogonal}, $(G)_{ij} = \langle v_i, v_j \rangle$ ($i \neq j$), and $|(G)_{ij}|=\mathcal{O}(1/\sqrt{d})$.  $G$ can be written as:
\begin{equation}
G = I_V + E,
\end{equation}
where $I_V$ is the identity matrix, $E$ is the full interference matrix with diagonal elements equal to 0 and off-diagonal absolute values equal to $\mathcal{O}(1/\sqrt{d})$, i.e., 
\begin{equation}
E=\begin{cases}
|E_{ij}|
=0,\quad &i=j\\
|E_{ij}|= \mathcal{O}\left(\frac{1}{\sqrt{d}}\right), \quad &i\neq j
\end{cases}.
\end{equation}
Therefore, the target matrix $M$ becomes:
\begin{equation}
M = \Lambda_p^{1/2} (I_V + E) \Lambda_p^{1/2} = \Lambda_p + \Lambda_p^{1/2} E \Lambda_p^{1/2}.
\end{equation}
Let:
\begin{equation}
A = \Lambda_p = \text{diag}(p_1, \dots, p_V), \quad \Delta = \Lambda_p^{1/2} E \Lambda_p^{1/2}.
\end{equation}
Since $A$ is a diagonal matrix, the eigenvalues of $A$ are $p_k$. By Theorem~\ref{lemma:weyl}, the difference between $\tau_k$ and $p_k$ is bounded by the 2-norm of perturbation matrix $\Delta$:
\begin{equation}
\label{eq:weyl}
|\tau_k - p_k| \le \|\Delta\|_2.
\end{equation}
We need to compute the upper bound of $\|\Delta\|_2$, where
\begin{equation}
    \Delta=\begin{cases}
(\Delta)_{ij} = \sqrt{p_i p_j} \cdot \langle v_i, v_j \rangle, \quad &i\neq j\\
(\Delta)_{ij} = 0, \quad &i=j
    \end{cases}.
\end{equation}
Using the matrix norm inequality:
\begin{equation}
\|\Delta\|_2 \le \|\Delta\|_F = \sqrt{\sum_{i \neq j} p_i p_j \langle v_i, v_j \rangle^2},
\end{equation}
where $||\cdot||_F$ is the Frobenius norm. By~\cref{quasi_orthogonal}, $|\langle v_i, v_j \rangle| =\mathcal{O}(1/\sqrt{d})$. Then, 
\begin{equation}
\|\Delta\|_2 \le \mathcal{O}\left(\frac{1}{\sqrt{d}}\right) \cdot \sqrt{\sum_{i \neq j} p_i p_j} < \mathcal{O}\left(\frac{1}{\sqrt{d}}\right) \cdot \sum_{i=1}^V p_i.
\end{equation}
Due to probability normalization $\sum p_i = 1$:
\begin{equation}
\label{eq:weylTwo}
\|\Delta\|_2 \le\mathcal{O}\left(\frac{1}{\sqrt{d}}\right) \cdot 1 = \mathcal{O}\left(\frac{1}{\sqrt{d}}\right).
\end{equation}
By~\cref{eq:weyl} and~\cref{eq:weylTwo}, 
\begin{equation}
| \tau_k - p_k | \le \mathcal{O}\left(\frac{1}{\sqrt{d}}\right).
\end{equation}
Note that $d\rightarrow \infty$ leads to $\mathcal{O}(1/\sqrt{d})\rightarrow 0$. Thus, for a fixed $p_k$:
\begin{equation}
\frac{| \tau_k - p_k |}{p_k} \to 0 \quad ( d \to \infty).
\end{equation}
By Definition~\ref{AsymEqui}, the asymptotic equivalence relation holds:
\begin{equation}
\tau_k \sim p_k.
\end{equation}
By Assumption~\ref{Zipf_law}, the following relation holds:
\begin{equation}
\tau_k\sim p_k \sim k^{-\alpha}.
\end{equation}
\end{proof}

\begin{proof}[Proof of Lemma~\ref{LemmaPhaseTransition}]
    By Cauchy's integral test~\cite{cauchy}, for a monotonically decreasing function $f(x) = x^{-\alpha}$, there is an inequality holds:
    \begin{equation}
\sum_{j=r+1}^d j^{-\alpha} < \int_{r}^{d} x^{-\alpha} dx.
\end{equation}
Calculate the definite integral:
\begin{equation}
\int_{r}^{d} x^{-\alpha} dx = \left[ \frac{x^{1-\alpha}}{1-\alpha} \right]_{r}^{d} = \frac{1}{\alpha-1} \left( r^{1-\alpha} - d^{1-\alpha} \right).
\end{equation}
Since $d > r$ and $\alpha > 1$, $d^{1-\alpha} > 0$.  Discard the subtraction term $-d^{1-\alpha}$ for further scale:
\begin{equation}
\label{eq:integral}
\int_{r}^{d} x^{-\alpha} dx < \frac{r^{1-\alpha}}{\alpha-1}.
\end{equation}
Substituting~\cref{eq:integral} into~\cref{eq:NoiseDef} leads to: 
\begin{equation}
\label{eq:cauchyBound}
\nu^2(d) < \frac{L}{d-r} \cdot \frac{r^{1-\alpha}}{\alpha-1}.
\end{equation}
For a fixed head rank $k\in [1,r]$, define a signal-to-noise ratio function $R(d)$ as: 
\begin{equation}
R(d) \triangleq \frac{\tau_k}{\nu^2(d)}.
\end{equation}
By~\cref{eq:cauchyBound}:
\begin{equation}
\label{eq:RBound}
\begin{split}
R(d) = \frac{L \cdot k^{-\alpha}}{\nu^2(d)} > \frac{L \cdot k^{-\alpha}}{\left( \frac{L}{d-r} \cdot \frac{r^{1-\alpha}}{\alpha-1} \right)}=\frac{k^{-\alpha}(\alpha-1)}{r^{1-\alpha}} \cdot (d-r).
\end{split}
\end{equation}
Since $r$ and $\frac{k^{-\alpha}(\alpha-1)}{r^{1-\alpha}}$ are positive real number, the following limit holds:
\begin{equation}
\lim_{d \to \infty}  \frac{k^{-\alpha}(\alpha-1)}{r^{1-\alpha}} \cdot (d-r) = +\infty.
\end{equation}
By comparison test,~\cref{eq:RBound} leads to 
\begin{equation}
\lim_{d \to \infty} R(d) = +\infty.
\end{equation}
Due to $\lim_{d \to \infty} R(d)=\lim_{d \to \infty} \frac{\tau_k}{\sigma^2(d)} = +\infty$, for any given constant $1+\sqrt{c}$, there is a  dimension $d_1$ such that for all $d > d_1$, the following relation holds:
\begin{equation}
\frac{\tau_k}{\nu^2(d)} > 1+\sqrt{c}.
\end{equation}
Namely:
\begin{equation}
\label{eq:head}
\tau_k > \nu^2(d)\cdot (1+\sqrt{c}).
\end{equation}
On the other hand, $\lim_{d \to \infty}\tau_k=0$ when $k\in (r,d]$. For any given constant $1+\sqrt{c}$, there is a dimension $d_2$, such that for all $d>d_2$, the following relation holds:
\begin{equation}
\label{eq:tail}
\tau_k \leq \nu^2(d)\cdot (1+\sqrt{c}).
\end{equation}
Take $d=\max\{d_1, d_2\}$, then~\cref{eq:head} and~\cref{eq:tail} hold simultaneously. Furthermore, recall the definition of $\nu^2(d)$:
\begin{equation}
\nu^2(d) = \frac{L}{d-r} \sum_{j=r+1}^d j^{-\alpha}.
\end{equation}
When $d\rightarrow \infty$, $\sum_{j=r+1}^d j^{-\alpha}$ can be regarded as the tail of Riemann Zeta function $\zeta(\alpha)$~\cite{RiemannZeta}:
\begin{equation}
\zeta(\alpha)=\sum_{j=1}^{\infty}j^{-\alpha}.
\end{equation}
Since $\zeta(\alpha)$ converges, $\sum_{j=r+1}^d j^{-\alpha}$ converges. As $L,r$ are constants, there are constants $C>0, C,d_0\in\mathbb{R}$ such that for all $d>d_0$ the following condition is satisfied
\begin{equation}
|\nu^2(d)| = \left|\frac{L}{d-r} \sum_{j=r+1}^d j^{-\alpha}\right|\leq C\cdot\left|\frac{1}{d}\right|.
\end{equation}
By Definition~\ref{AsymBound}, that is $\nu^2(d)=\mathcal{O}(d^{-1})$.
\end{proof}

\begin{proof}[Proof of Lemma~\ref{LemmaSTrans}]
    For any $a$ and $z$ ($a \neq z$), the following identity holds:
    \begin{equation}
    \label{eq:identity}
\frac{1}{a - z} = -\frac{1}{z} - \frac{a}{z^2} + \frac{a^2}{z^2(a-z)}.
\end{equation}
Substituting $a = \tau_j$ into~\cref{eq:identity} and expanding $m_{Noise}(z)$ gives:
\begin{equation}
\label{eq:Noise}
\begin{aligned} m_{Noise}(z) &= \frac{1}{d-r} \sum_{j=r+1}^d \left( -\frac{1}{z} - \frac{\tau_j}{z^2} + \frac{\tau_j^2}{z^2(\tau_j - z)} \right) \\ &= -\frac{1}{z} - \frac{1}{z^2} \left( \frac{1}{d-r} \sum_{j=r+1}^d \tau_j \right) + \frac{1}{d-r} \sum_{j=r+1}^d \frac{\tau_j^2}{z^2(\tau_j - z)}\\
&= -\frac{1}{z} - \frac{\nu^2(d)}{z^2} + R_{Noise}(z)
\end{aligned},
\end{equation}
where
\begin{equation}
R_{Noise}(z)\triangleq\frac{1}{d-r} \sum_{j=r+1}^d \frac{\tau_j^2}{z^2(\tau_j - z)}.
\end{equation}
Substituting $a = \nu^2(d)$ into~\cref{eq:identity} and expanding $m_{White}(z)$ gives:
\begin{equation}
\label{eq:white}
\begin{aligned} m_{White}(z) &= -\frac{1}{z} - \frac{\nu^2(d)}{z^2} + \frac{(\nu^2(d))^2}{z^2(\nu^2(d) - z)}
\end{aligned}.
\end{equation}
Denote
\begin{equation}
 R_{White}(z)\triangleq\frac{(\nu^2(d))^2}{z^2(\nu^2(d) - z)}.
\end{equation}
Now examine the difference $\Delta(z) \triangleq |m_{Noise}(z) - m_{White}(z)|$. The first two terms of~\cref{eq:Noise} and~\cref{eq:white} are the same, so by the triangle inequality:
\begin{equation}
\label{eq:noiseB}
\Delta(z) = |R_{Noise}(z) - R_{White}(z)| \le |R_{Noise}(z)| + |R_{White}(z)|.
\end{equation}
For a fixed $z$, since $\tau_j \in \{ \tau_j = L j^{-\alpha} \}_{j=r+1}^d$,
\begin{equation}
\frac{\tau_j^2}{z^2(\tau_j-z)}=\mathcal{O}(\tau_j).
\end{equation}
Therefore,
\begin{equation}
\label{eq:RNoise}
|R_{Noise}| =\mathcal{O}\left(\frac{1}{d-r} \sum_{j=r+1}^d \tau_j\right)=\mathcal{O}\left(\frac{L}{d-r} \sum_{j=r+1}^d j^{-\alpha}\right).
\end{equation}
When $d\rightarrow \infty$ in~\cref{eq:RNoise}, $\sum_{j=r+1}^d j^{-\alpha}$ can be regarded as the tail of Riemann Zeta function $\zeta(\alpha)$. Since $\zeta(\alpha)$ converges, $\sum_{j=r+1}^d j^{-\alpha}$ converges. Thus,
\begin{equation}
\label{eq:x1}
\frac{L}{d-r} \sum_{j=r+1}^d j^{-\alpha}=\mathcal{O}(d^{-1})
\end{equation}
~\cref{eq:Noise} and~\cref{eq:x1} give:
\begin{equation}
\label{eq:NoiseEq}
|R_{Noise}| = \mathcal{O}(\mathcal{O}(d^{-1}))=\mathcal{O}(d^{-1}).
\end{equation}
Similarly,
\begin{equation}
|R_{White}| =\mathcal{O} (\nu^2(d)).
\end{equation}
Lemma~\ref{LemmaPhaseTransition} shows that $\nu^2(d)= \mathcal{O}(d^{-1})$, so 
\begin{equation}
\label{eq:WhiteEq}
|R_{White}| = \mathcal{O}(\mathcal{O}(d^{-1}))=\mathcal{O}(d^{-1}).
\end{equation}
Substituting~\cref{eq:NoiseEq} and~\cref{eq:WhiteEq} into~\cref{eq:noiseB} gives:
\begin{equation}
\lim_{d \to \infty} \Delta(z) = \lim_{d \to \infty} [\mathcal{O}(d^{-1}) + \mathcal{O}(d^{-1})] = 0
\end{equation}
By the Squeeze theorem,
\begin{equation}
\lim_{d \to \infty}|m_{Noise}(z) - m_{White}(z)|=0
\end{equation}

\end{proof}

\begin{proof}[Proof of Lemma~\ref{LemmaESD}]
Examine the normalization matrix $\tilde{\Sigma} = \frac{1}{\nu^2(d)} \Sigma$:
\begin{equation}
\tilde{\Sigma} = \frac{1}{\nu^2(d)} \Sigma_{Signal} + \frac{1}{\nu^2(d)} \Sigma_{Noise}.
\end{equation}
Define the set of  eigenvalues belonging to $\frac{1}{\nu^2(d)} \Sigma_{Noise}$
\begin{equation}
\tilde{\mathcal{T}}_A=\{\tilde{\tau}_j = \frac{\tau_j}{\nu^2(d)}\}_{j=r+1}^d.
\end{equation}
The Stieltjes transform of $\tilde{\mathcal{T}}_A$ is:
\begin{equation}
\tilde{m}_{Noise}(z) \triangleq \frac{1}{d-r} \sum_{j=r+1}^d \frac{1}{\tilde{\tau}_j - z}.
\end{equation}
Substituting $\tilde{\tau}_j$'s definition and multiplying both the numerator and denominator by $\nu^2(d)$ give:
\begin{align}
\tilde{m}_{Noise}(z) &= \frac{1}{d-r} \sum_{j=r+1}^{d} \frac{1}{\frac{\tau_j}{\nu^2(d)} - z} \\
&= \frac{1}{d-r} \sum_{j=r+1}^{d} \frac{\nu^2(d)}{\tau_j - \nu^2(d)z} \\
&= \nu^2(d) \cdot \underbrace{\left( \frac{1}{d-r} \sum_{j=r+1}^{d} \frac{1}{\tau_j - \nu^2(d)z} \right)}_{m_{Noise}(\nu^2(d)z)}
\end{align}
% \begin{equation}
% \begin{split}
% \tilde{m}_{Noise}(z) &= \frac{1}{d-r} \sum_{j=r+1}^d \frac{1}{\frac{\tau_j}{\nu^2(d)} - z}\\
% &= \frac{1}{d-r} \sum_{j=r+1}^d \frac{\nu^2(d)}{\tau_j - \nu^2(d) z} \\
% &= \nu^2(d) \cdot \underbrace{\left( \frac{1}{d-r} \sum_{j=r+1}^d \frac{1}{\tau_j - (\nu^2(d) z)} \right)}_{m_{Noise}(\nu^2(d) z)}.
% \end{split}
% \end{equation}
The remaining summation term is $m_{Noise}(z)$, with the independent variable becoming $\nu^2(d) z$. Therefore, the following relation holds:
\begin{equation}
\tilde{m}_{Noise}(z) = \nu^2(d) \cdot m_{Noise}(\nu^2(d) z)
\end{equation}
By Lemma~\ref{LemmaSTrans} $\lim_{d \to \infty}|m_{Noise}(z) - m_{White}(z)|=0$, so the relation holds:
\begin{align}
\label{eq:mlimit}
 \lim_{d \to \infty} \tilde{m}_{Noise}(z) &=\lim_{d \to \infty}\nu^2(d) \cdot \left( \frac{1}{\nu^2(d) - (\nu^2(d) z)} \right) \\ &= \lim_{d \to \infty}\nu^2(d) \cdot \frac{1}{\nu^2(d) (1 - z)}  \\ &= \lim_{d \to \infty}\frac{\nu^2(d)}{\nu^2(d)} \cdot \frac{1}{1 - z} \\ &= \lim_{d \to \infty}\frac{1}{1 - z}\\
&= \frac{1}{1 - z}.
\end{align}
% Note that function $f(z)=\frac{1}{1-z}$ is the Stieltjes transform of the Dirac measure $\delta_1$. In fact,
% \begin{equation}
% m_{\delta_1}(z) = \int_{\mathbb{R}} \frac{1}{t - z} d\delta_1(t)=m_{\delta_1}(z) = \frac{1}{1 - z}
% \end{equation}
% By Theorem~\ref{ThmStieltjesTransformContinuity}, the ESD of $\frac{1}{\nu^2(d)} \Sigma_{Noise}$ weakly converges to $\delta_1$:
% \begin{equation}
% \text{ESD}(\frac{1}{\nu^2(d)} \Sigma_{Noise})\xrightarrow{w.} \delta_1 
% \end{equation}

% Theorem~\ref{ThmStieltjesTransformContinuity} states that if a sequence of Stieltjes transforms $m_n(z)$ converges to a limit $m(z)$, and $m(z)$ corresponds to a unique probability measure $\mu$, then the underlying distributions $\mu_n$ must weakly converge to $\mu$. \cref{eq:mlimit} gives the asymptotic limit of the noise matrix's transform:
%     \begin{equation}
%         \lim_{d \to \infty} \tilde{m}_{Noise}(z) = \frac{1}{1-z}.
%     \end{equation}
%     Note that the limit function $(1-z)^{-1}$ is the unique Stieltjes transform of the Dirac mass $\delta_1$.  In fact,
% \begin{equation}
% m_{\delta_1}(z) = \int_{\mathbb{R}} \frac{1}{t - z} d\delta_1(t)=m_{\delta_1}(z) = \frac{1}{1 - z}
% \end{equation}
% Consequently, applying Theorem~\ref{ThmStieltjesTransformContinuity} allows us to translate the algebraic limit directly into the geometric collapse of the spectrum:
%     \begin{equation}
%         \text{ESD}\left(\frac{1}{\nu^2(d)} \Sigma_{Noise}\right) \xrightarrow{w} \delta_1.
%     \end{equation}

Theorem~\ref{ThmStieltjesTransformContinuity} states that if a sequence of Stieltjes transforms $m_n(z)$ converges to a limit $m(z)$, and $m(z)$ corresponds to a unique probability measure $\mu$, then the underlying distributions $\mu_n$ must weakly converge to $\mu$. \cref{eq:mlimit} gives the asymptotic limit of the noise matrix's transform:
\begin{equation}
    \lim_{d \to \infty} \tilde{m}_{Noise}(z) = \frac{1}{1-z}.
\end{equation}
Note that the limit function $(1-z)^{-1}$ is the unique Stieltjes transform of the Dirac mass $\delta_1$. In fact,
\begin{equation}
    m_{\delta_1}(z) = \int_{\mathbb{R}} \frac{1}{t - z} d\delta_1(t) = \frac{1}{1 - z}.
\end{equation}
Consequently, identifying the ESD as the sequence of measures $\{\mu_n\}$, we invoke Theorem~\ref{ThmStieltjesTransformContinuity} to translate the algebraic limit directly into the geometric collapse of the spectrum:
\begin{equation}
    \text{ESD}\left(\frac{1}{\nu^2(d)} \Sigma_{Noise}\right) \xrightarrow{w} \delta_1.
\end{equation}

By the linear property of random variables, the ESD of $\Sigma_{Noise}$ weakly converges to the Dirac measure concentrated at $\nu^2(d)$, i.e.,
\begin{equation}
\text{ESD}( \Sigma_{Noise})\xrightarrow{w.} \delta_{\nu^2(d)}, \quad  d\rightarrow \infty 
\end{equation}
\end{proof}
\begin{proof}[Proof of Theorem~\ref{ThmBBP}]

By Lemma~\ref{LemmaESD}, $\text{ESD}(\Sigma_{Noise})$ weakly converges to the Dirac measure $\delta_{\nu^2(d)}$:
\begin{equation}
\text{ESD}( \Sigma_{Noise})\xrightarrow{w.} \delta_{\nu^2(d)}, \quad  d\rightarrow \infty 
\end{equation}
Obviously, $\delta_{\nu^2(d)}$ is a deterministic probability measure. Since $\nu^2(d)=\mathcal{O}(d^{-1})$, the support set $\{\nu^2(d)\}$ is a compact set. Therefore, take $\delta_{\nu^2(d)}$ as $\mu$. By Lemma~\ref{LemmaPhaseTransition}, for $d\rightarrow \infty, k=1,...,r$, the relation:
\begin{equation}
\tau_k > \nu^2(d)\cdot (1+\sqrt{c})
\end{equation}
holds, which means that $\text{Spec}(\Sigma_{signal}) = (\tau_1, \dots, \tau_r, 0, \dots, 0)$ falls into the super-critical regime. For $d\rightarrow \infty, k=r+1,...,d$, the relation:
\begin{equation}
\tau_k \leq \nu^2(d)\cdot (1+\sqrt{c})
\end{equation}
holds, which means that $\text{Spec}(\Sigma_{tail})=(0, \dots, 0, \tau_{r+1}, \dots, \tau_d)$ falls into the sub-critical regime.
\end{proof}

\begin{proof}[Proof of Lemma~\ref{LemmaLambdaEqui}]
By Theorem~\ref{ThmBBP}, since $\text{Spec}(\Sigma_{Signal})$ falls into the super-critical regime, the relation holds:
\begin{equation}
\label{eq:rho}
 \lambda_k \xrightarrow{a.s.}\rho(\tau_k)= \tau_k \left( 1 + \frac{c \cdot \nu^2(d)}{\tau_k - \nu^2(d)} \right). 
\end{equation}
Notice that:
\begin{equation}
\frac{1}{\tau_k - \nu^2(d)} = \frac{1}{\tau_k \left( 1 - \frac{\nu^2(d)}{\tau_k} \right)} = \frac{1}{\tau_k} \cdot \left( 1 - \frac{\nu^2(d)}{\tau_k} \right)^{-1}
\end{equation}
Applying the power series expansion formula $(1-x)^{-1} = 1 + x + x^2 + \mathcal{O}(x^3)$, where $x = \frac{\nu^2(d)}{\tau_k}$, gives:
\begin{equation}
\label{eq:taylor}
\left( 1 - \frac{\nu^2(d)}{\tau_k} \right)^{-1} = 1 + \frac{\nu^2(d)}{\tau_k} + \left(\frac{\nu^2(d)}{\tau_k}\right)^2 + \mathcal{O}(\tau_k^{-3}).
\end{equation}
Substituting~\cref{eq:taylor} into~\cref{eq:rho} gives:
\begin{equation}
\begin{aligned} \rho(\tau_k) &= \tau_k \left[ 1 + c \nu^2(d) \cdot \left( \frac{1}{\tau_k} \left( 1 + \frac{\nu^2(d)}{\tau_k} + \frac{(\nu^2(d))^2}{\tau_k^2} + \dots \right) \right) \right] \\ &= \tau_k \left[ 1 + \frac{c \nu^2(d)}{\tau_k} + \frac{c (\nu^2(d))^2}{\tau_k^2} + \frac{c (\nu^2(d))^4}{\tau_k^3} + \mathcal{O}(\tau_k^{-4}) \right]\\
&= \tau_k + c \nu^2(d) + \frac{c (\nu^2(d))^2}{\tau_k} + \mathcal{O}(\tau_k^{-2})  \end{aligned}
\end{equation}
By the definition of almost surely convergence,
\begin{equation}
\lim_{\tau_k \to \infty} \frac{\lambda_k}{\tau_k} = \lim_{\tau_k \to \infty} \left( \frac{\tau_k + c \nu^2(d) + \frac{c (\nu^2(d))^2}{\tau_k} + O(\tau_k^{-2})}{\tau_k} \right) = 1
\end{equation}
By Definition~\ref{AsymEqui}, the relation holds:
\begin{equation}
\lambda_k\sim \tau_k, \quad k=1,...,r.
\end{equation}
Similarly, since $\text{Spec}(\Sigma_{tail})$ falls into the sub-critical regime of Theorem~\ref{ThmBBP}, 
\begin{equation}
\lambda_k\sim \nu^2(d)(1+\sqrt{c})^2, \quad k=r+1,...,r.
\end{equation}
\end{proof}

\begin{proof}[Proof of Theorem~\ref{ThmX}]
Let the SVD of $\mathbf{X}$ be $\mathbf{X} = U \Sigma_{svd} V^T$, where $\Sigma_{svd}$ is a diagonal matrix with diagonal elements being singular values $\sigma_1(\mathbf{X}) \ge \sigma_2(\mathbf{X}) \ge \dots \ge \sigma_d(\mathbf{X}) \ge 0$. Recall that the sample covariance matrix is defined as:
\begin{equation}
S_N=\frac{1}{N}\mathbf{X}\mathbf{X}^T\in \mathbb{R}^{d\times d}.
\end{equation}
Substituting $\mathbf{X} = U \Sigma_{svd} V^T$ into $S_N$ yields:
\begin{equation}
\begin{aligned}
S_N &= \frac{1}{N} (U \Sigma_{svd} V^T) (U \Sigma_{svd} V^T)^T\\
&=\frac{1}{N} U \Sigma_{svd} V^T V \Sigma_{svd}^T U^T
\end{aligned}
\end{equation}
Since $V$ is an orthogonal matrix and $\Sigma_{svd}$ is a diagonal matrix, $V^T V = I$,  $\Sigma_{svd}^T = \Sigma_{svd}$. Therefore:
\begin{equation}
\label{SnSVD}
\begin{aligned}
S_N &=\frac{1}{N} U (\Sigma_{svd} \Sigma_{svd}) U^T\\
&=U \left( \frac{1}{N} \Sigma_{svd}^2 \right) U^T
\end{aligned}
\end{equation}
Note that~\cref{SnSVD} is the SVD of $S_N$, where $U$ is the eigenvector matrix.  Therefore,
\begin{equation}
\lambda_k = \frac{1}{N} \sigma_k^2(\mathbf{X}),
\end{equation}
which gives:
\begin{equation}
\sigma_k(\mathbf{X}) = \sqrt{N \cdot \lambda_k}.
\end{equation}
Since $N$ is a scalar constant, in the limit process of $d, N \to \infty$, it only affects the amplitude and does not affect the shape of the distribution. By Lemma~\ref{LemmaLambdaEqui},
\begin{equation}
\begin{aligned} \sigma_k(\mathbf{X}) &\sim \sqrt{N \cdot \tau_k} =
\begin{cases}
\sqrt{NL}\cdot k^{-\alpha/2},& \quad k\leq r\\

\sqrt{N\nu^2(d)}\cdot(1+\sqrt{c}),&\quad r<k\leq d
\end{cases}
\end{aligned}.
\end{equation}
By Definition~\ref{AsymBound}, $\sqrt{NL}\cdot k^{-\alpha/2}\sim \mathcal{O}(k^{-\alpha/2})$. Recall  the definition of $\nu^2(d)$:
\begin{equation}
\nu^2(d)= \frac{1}{d-r} \sum_{j=r+1}^d \tau_j = \frac{L}{d-r} \sum_{j=r+1}^d j^{-\alpha}.
\end{equation}
Substituting the definition of $\nu^2(d)$ into $\sqrt{N\nu^2(d)}\cdot(1+\sqrt{c})$, when $d,N\rightarrow \infty$,
\begin{equation}
\label{eq:OBound}
\begin{aligned}
\sqrt{N\nu^2(d)}\cdot(1+\sqrt{c})&\rightarrow(1+\sqrt{\frac{d}{N}})\cdot\sqrt{\frac{NL}{d-r}\sum_{j=r+1}^dj^{-\alpha}}\\
&=\sqrt{\frac{NL}{d-r}\sum_{j=r+1}^dj^{-\alpha}}+\sqrt{\frac{dL}{d-r}\sum_{j=r+1}^dj^{-\alpha}}
\end{aligned}
\end{equation}
Since $\sum_{j=r+1}^dj^{-\alpha}$ is the tail of the Riemann Zeta function $\zeta(\alpha)$, it converges when $d\rightarrow \infty$. Therefore,
\begin{equation}
\label{eq:OBounddd}
    \sqrt{\frac{NL}{d-r}\sum_{j=r+1}^dj^{-\alpha}}=\mathcal{O}(d^{-\frac{1}{2}}),\quad \sqrt{\frac{dL}{d-r}\sum_{j=r+1}^dj^{-\alpha}}=\mathcal{O}(1).
\end{equation}
Substituting~\cref{eq:OBounddd} into~\cref{eq:OBound} yields:
\begin{equation}
\sqrt{N\nu^2(d)}\cdot(1+\sqrt{c})\sim \mathcal{O}(1).
\end{equation}
Therefore,
\begin{equation}
\begin{aligned} \sigma_k(\mathbf{X}) &\sim \
\begin{cases}
 \mathcal{O}(k^{-\alpha/2}),& \quad 1\leq k\leq r\\

\mathcal{O}(1),&\quad r<k\leq d
\end{cases}
\end{aligned}.
\end{equation}

By Definition~\ref{AsymBound}, there are two constants $C_1, C_2$, such that:
\begin{equation}
\begin{aligned} \sigma_k(\mathbf{X}) &\leq \
\begin{cases}
 C_1 \cdot k^{-\alpha/2},& \quad 1\leq k\leq r\\
C_2 ,&\quad r<k\leq d
\end{cases}
\end{aligned}.
\end{equation}
Since $ \sigma_{k+1}(\mathbf{X})\leq  \sigma_k(\mathbf{X})$, for all $k\in(r,d]$, 
\begin{equation}
 \sigma_k(\mathbf{X})\leq  \sigma_r(\mathbf{X})\leq  C_1\cdot r^{-\alpha/2}
\end{equation}
Take $C_2=C_1\cdot r^{-\alpha/2}$ yields:
\begin{equation}
\begin{aligned} \sigma_k(\mathbf{X}) &\leq \
\begin{cases}
 C_1 \cdot k^{-\alpha/2},& \quad 1\leq k\leq r\\
C_1\cdot r^{-\alpha/2} ,&\quad r<k\leq d
\end{cases}
\end{aligned}
\end{equation}
\end{proof}

\begin{proof}[Proof of Theorem~\ref{ThmG}]
By the Chain Rule, the gradient of the weight matrix $\nabla_\mathbf{W}$ is calculated by the product of the transpose of the input matrix and the error signal matrix:
\begin{equation}
\nabla_\mathbf{W} = \mathbf{X}^\top \mathbf{G},
\end{equation}
where $\mathbf{X}^\top\in\mathbb{R}^{N\times d}$.  By the product property of singular values of matrices, for any two matrices $A$ and $B$, the $k$-th largest singular value of $AB$ is bounded by: 
\begin{equation}
\sigma_k(AB) \le \|A\|_2 \cdot \sigma_k(B).
\end{equation}
As the transposition does not change matrix singular values, 
\begin{equation}
\sigma_k(\nabla_\mathbf{W})=\sigma_k(\nabla_\mathbf{W}^\top) = \sigma_k(\mathbf{G^\top X}) \le \|\mathbf{G}^\top\|_2 \cdot \sigma_k(\mathbf{X}).
\end{equation}
Since $\|\mathbf{G}^\top\|_2=\|\mathbf{G}\|_2$, 
\begin{equation}
\label{eq:WBound}
\sigma_k(\nabla_\mathbf{W}) \le \|\mathbf{G}\|_2 \cdot \sigma_k(\mathbf{X})\leq M\cdot \sigma_k(\mathbf{X}) 
\end{equation}
By Theorem~\ref{ThmX}, the $k$-th largest singular value of $\mathbf{X}$ satisfies the following asymptotic behavior:
\begin{equation}
\label{eq:XBound}
\begin{aligned} \sigma_k(\mathbf{X}) &\leq \
\begin{cases}
 C_1 \cdot k^{-\alpha/2},& \quad 1\leq k\leq r\\
C_1\cdot r^{-\alpha/2} ,&\quad r<k\leq d
\end{cases}
\end{aligned},
\end{equation}
where $C_1$ is a constant. Substituting~\cref{eq:XBound} into~\cref{eq:WBound} yields:
\begin{equation}
\begin{aligned} \sigma_k(\nabla_\mathbf{W}) &\leq \
\begin{cases}
(MC_1) \cdot k^{-\alpha/2},& \quad 1\leq k\leq r\\
(MC_1)\cdot r^{-\alpha/2} ,&\quad r<k\leq d
\end{cases}
\end{aligned}
\end{equation}
    
\end{proof}

\section{Math Proof of Quantization Increasing Stable Rank}
\label{sec:quant_increase_sr}

\begin{proof}[Proof of~\cref{LemmaCrossTerms}]
We provide the proof for the first equality; the second follows by symmetry. By definition, substituting $E = \sum_{i,j} Z_{ij}$ into the expectation yields:
\begin{align}
    \mathbb{E}[EE^T] &= \mathbb{E}\left[ \left(\sum_{i,j} Z_{ij}\right) \left(\sum_{p,q} Z_{pq}\right)^T \right] \\
    &= \mathbb{E}\left[ \sum_{i,j} \sum_{p,q} Z_{ij} Z_{pq}^T \right].
\end{align}
Exploiting the linearity of the expectation operator, we move the expectation inside the double summation:
\begin{equation}
    \mathbb{E}[EE^T] = \sum_{i,j} \sum_{p,q} \mathbb{E}[Z_{ij} Z_{pq}^T].
\end{equation}
We analyze the summand $\mathbb{E}[Z_{ij} Z_{pq}^T]$ based on the index pairs $(i,j)$ and $(p,q)$. When $(i,j) \ne (p,q)$, due to the mutual independence of the quantization noise, the expectation factorizes: 
\begin{equation}
    \mathbb{E}[Z_{ij} Z_{pq}^T] = \mathbb{E}[Z_{ij}] \mathbb{E}[Z_{pq}^T].
\end{equation}
Since $\mathbb{E}[Z_{ij}] = 0$, these terms vanish identically ($0 \cdot 0 = 0$). When $(i,j) = (p,q)$, these terms correspond to the self-variance and remain as
\begin{equation}
    \mathbb{E}[Z_{ij} Z_{ij}^T].
\end{equation}
Consequently, all off-diagonal cross-terms in the summation are zero, and the expression collapses to a single summation over the basis matrices:
\begin{equation}
    \mathbb{E}[EE^T] = \sum_{i,j} \mathbb{E}[Z_{ij} Z_{ij}^T].
\end{equation}
Similarly, $\mathbb{E}[E^TE] = \sum_{i,j} \mathbb{E}[Z_{ij}^T Z_{ij}]$ holds following the above analysis. This concludes the proof of Lemma~\ref{LemmaCrossTerms}.
\end{proof}

\begin{proof}[Proof of~\cref{ThmBernstein2}]
With Lemma~\ref{LemmaCrossTerms} established, we proceed to compute the value of $\delta^2$. We evaluate the matrix $\mathbb{E}[EE^T]$ element-wise. The $(i,j)$-th entry of the product matrix $EE^T$ is given by the inner product of the $i$-th and $j$-th rows of $E$:
\begin{equation}
    (EE^T)_{ij} = \sum_{k=1}^n e_{ik}e_{jk}.
\end{equation}
Taking the expectation of this entry and applying linearity:
\begin{equation}
    \mathbb{E}[(EE^T)_{ij}] = \sum_{k=1}^n \mathbb{E}[e_{ik}e_{jk}].
\end{equation}
We again invoke the independence property of the noise terms $e_{ik}$ and $e_{jk}$. For off-diagonal entries where $i \ne j$, the noise terms are independent, so 
\begin{equation}
    \mathbb{E}[e_{ik}e_{jk}] = \mathbb{E}[e_{ik}]\mathbb{E}[e_{jk}] = 0.
\end{equation}
Thus, $\mathbb{E}[EE^T]$ is a diagonal matrix. For diagonal entries where $i = j$, the expectation becomes the variance sum: $\mathbb{E}[e_{ik}^2] = B$. The summation yields:
\begin{equation}
        \sum_{k=1}^n \mathbb{E}[e_{ik}^2] = \sum_{k=1}^n B = nB.
\end{equation}

Therefore, $\mathbb{E}[EE^T]$ is a scaled identity matrix of size $m \times m$ with diagonal values $nB$. Its spectral norm is simply this diagonal value:
\begin{equation}
    \left\| \sum_{i,j} \mathbb{E}[Z_{ij} Z_{ij}^T] \right\|_2 = nB.
\end{equation}
By a symmetric argument for $\mathbb{E}[E^T E]$, we obtain a norm of $mB$. The variance statistic $\delta^2$ is determined by the maximum of these two values:
\begin{equation}
    \delta^2 = \max \left( nB, mB \right) = \max(m, n) \cdot B.
\end{equation}

Substituting this specific $\delta^2$ into Theorem \ref{thm:bernstein}, we obtain the tail bound for the quantization error. For large-scale deep learning models where $m, n \gg 1$, the variance term $\delta^2$ dominates the boundedness term $Rt/3$ in the denominator. Ignoring higher-order small terms, there exists a constant $C$ such that the bound simplifies to:
\begin{equation}
    P(\|E\|_2 \ge t) \le (m+n) \cdot \exp \left( -\frac{t^2}{C \cdot \max(m,n) \cdot s^2} \right),
\end{equation}
which concludes the proof.    
\end{proof}

\begin{proof}[Proof of~\cref{thm:relative_error}]
Our objective is to derive a probabilistic upper bound for the relative error, such that the stability holds with probability at least $1-\theta$. Recall that $\mathcal{R}(\epsilon_k)$ denotes the event where the relative error exceeds a threshold $\epsilon_k$. We algebraically transform this violation condition into an absolute error inequality:
\begin{equation}
\begin{split}
\mathcal{R}(\epsilon_k) &\iff \left\{ \frac{|\sigma_k(\tilde{A}) - \sigma_k(A)|}{\sigma_k(A)} > \epsilon_k \right\}\\
&\iff \left\{ |\sigma_k(\tilde{A}) - \sigma_k(A)| > \epsilon_k \cdot\sigma_k(A) \right\}.
\end{split}
\end{equation}
By applying Theorem \ref{lemma:weyl}, the spectral norm of the noise $\|E\|_2$ is a deterministic upper bound on the absolute perturbation $|\sigma_k(\tilde{A}) - \sigma_k(A)|$. This leads to a crucial logical implication: if the absolute perturbation exceeds $\epsilon_k\cdot \sigma_k(A)$, then the noise norm $\|E\|_2$ must necessarily exceed this threshold as well. In set-theoretic terms, the relative error violation event is a subset of the spectral norm tail event:
\begin{equation}
    \left\{ |\sigma_k(\tilde{A}) - \sigma_k(A)| > \epsilon_k \cdot \sigma_k(A) \right\} \subseteq \left\{ \|E\|_2 > \epsilon_k \cdot\sigma_k(A) \right\}.
\end{equation}
Consequently, by the monotonicity of probability measures, the probability of the relative error violation is strictly bounded by the tail probability of the spectral norm:
\begin{equation}
    P(\mathcal{R}(\epsilon_k)) \le P(\|E\|_2 > \epsilon_k\cdot \sigma_k(A)).
\end{equation}
Recalling the tail bound function $F(t) = P(\|E\|_2 \geq t)$ derived in Equation~\eqref{F(t)}, we obtain:
\begin{equation}
    P(\mathcal{R}(\epsilon_k)) \le F(\epsilon_k\cdot \sigma_k(A)).
\end{equation}
To guarantee the specified confidence level $1-\theta$, we must ensure that the violation probability  does not exceed $\theta$. We strictly equate the upper bound to $\theta$:
\begin{equation}
    F(\epsilon_k\cdot \sigma_k(A)) = \theta.
\end{equation}
This equation defines the boundary of our confidence interval. Since the tail function $F(t)$ is strictly monotonically decreasing for $t>0$, its inverse function $F^{-1}$ exists. Solving for the specific relative error threshold $\epsilon_k$ yields:
\begin{equation}
    \epsilon_k \cdot \sigma_k(A) = F^{-1}(\theta) \implies \epsilon_k(\theta) \triangleq\epsilon_k = \frac{F^{-1}(\theta)}{\sigma_k(A)}.
\end{equation}
In this final expression, the numerator $F^{-1}(\theta)$ represents a constant noise budget determined solely by the matrix dimensions, noise variance, and the confidence level $\theta$. It is independent of the index $k$.
This explicitly demonstrates the inverse proportionality:
\begin{equation}
    \epsilon_k(\theta) \propto \frac{1}{\sigma_k(A)}.
\end{equation}
Therefore, for any two singular values with magnitudes $\sigma_i > \sigma_j > 0$, under the same confidence level $1-\theta$, it strictly holds that:
\begin{equation}
    \epsilon_i(\theta) < \epsilon_j(\theta).
\end{equation}
This concludes the proof.
\end{proof}

\begin{proof}[Proof of Theorem~\ref{thm:asymptotic_robustness_exact}]
The proof relies on perturbation theory and concentration inequalities. By Theorem~\ref{ThmBernstein2}, the probability that the noise norm $||E||_2$ exceeds a threshold $t$ is bounded by:
\begin{equation}
    P(\|E\|_2 \ge t) \le (m+n) \cdot \exp\left( -\frac{t^2}{C \cdot D \cdot s^2} \right).
\end{equation}
Based on Theorem~\ref{lemma:weyl}, the absolute perturbation is bounded by the noise norm:
\begin{equation}
|\sigma_k(\tilde{A}) - \sigma_k(A)| \le \|E\|_2. 
\end{equation}
Thus, the failure condition $R_k > \eta$ implies $\|E\|_2 > \eta\cdot \sigma_k(A)$. Substituting $t = \eta\cdot \sigma_k(A)$ into the concentration bound yields:
\begin{equation}
\label{theta_P}
    \theta_k = P(R_k > \eta) \le (m+n) \cdot \exp\left( -\frac{\eta^2\cdot \sigma_k^2(A)}{C \cdot D \cdot s^2} \right).
\end{equation}
We now evaluate Equation (\ref{theta_P}) by substituting the piecewise asymptotic spectral profile from Assumption~\ref{ass:power_law}.

\textbf{Case 1:}
In regime $1\leq k\leq r$, $\sigma_k(A) \sim \mu \cdot k^{-\alpha/2}$. Squaring this term yields $\sigma_k^2(A) \sim \mu^2\cdot k^{-\alpha}$. Substituting this into the exponent:
\begin{equation}
\label{independent_D}
\begin{split}
    \frac{\eta^2\cdot \sigma_k^2(A)}{C \cdot D \cdot s^2} &\sim \frac{\eta^2 \cdot (\mu \cdot k^{-\alpha/2} )^2}{C \cdot D \cdot s^2} \\
    &= \frac{\eta^2 \cdot \mu^2\cdot k^{-\alpha} }{C \cdot D \cdot s^2} \\
    &= \underbrace{\frac{\eta^2 \mu^2}{C Ds^2}}_{\text{Constant term}} \cdot \underbrace{k^{-\alpha}}_{\text{Decay term}}.
\end{split}
\end{equation}
Applying Equation~\eqref{independent_D} to Equation~\eqref{theta_P} yields:
\begin{equation}
    \theta_k \lesssim (m+n) \cdot \exp\left( - \frac{\eta^2 \mu^2}{CD s^2} \cdot \frac{1}{k^{\alpha}} \right),
\end{equation}
\textbf{Case 2:} In regime $r<k\leq \min(m,n)$, the singular values flatten to the level of the $r$-th component: $\sigma_k(A) \sim \mu \cdot r^{-\alpha/2}$. Consequently, $\sigma_k^2(A) \sim \mu^2\cdot r^{-\alpha}$. Substituting this into the exponent yields a constant bound independent of $k$:
\begin{equation}
    \theta_k \lesssim (m+n) \cdot \exp\left( - \frac{\eta^2 \mu^2}{CDs^2} \cdot \frac{1}{r^{\alpha}} \right).
    \label{eq:bound_tail}
\end{equation}
Combining both cases, the explicit failure probability is bounded by:
\begin{equation}
    \theta_k \lesssim (m+n) \cdot \exp\left( - \frac{\eta^2 \mu^2}{CDs^2} \cdot \xi_k \right),
    \label{eq:final_piecewise_bound}
\end{equation}
where
\begin{equation}
    \xi_k = \begin{cases} k^{-\alpha}, & 1 \le k \le r \\ r^{-\alpha}, & r < k \le \min(m,n) \end{cases},
\end{equation}
which concludes the proof.
\end{proof}

\begin{proof}[Proof of~\cref{thm:stable_rank_impact}]
To analyze the shift in stable rank, we introduce the energy concentration function $F_A(k)$, defined as the proportion of the total energy captured by the top-$k$ principal components:
\begin{equation}
\label{eq:concentration}
    F_A(k) \triangleq \frac{H_k(A)}{H_k(A) + T_k(A)},
\end{equation}
where $H_k(A) = \sum_{i=1}^k \sigma_i^2(A)$ and $T_k(A) = \sum_{j=k+1}^r \sigma_j^2(A)$ represent the head and tail energies, respectively.
For the quantized matrix $\tilde{A}$, the perturbed energies can be expressed as $H_k(\tilde{A})=(1+\alpha)H_k(A)$ and $T_k(\tilde{A})=(1+\beta)T_k(A)$, where $\alpha,\beta\geq -1$. Substituting these into~\cref{eq:concentration} yields:
\begin{equation}
\begin{split}
    F_{\tilde{A}}(k) &= \frac{(1+\alpha)H_k(A)}{(1+\alpha)H_k(A) + (1+\beta)T_k(A)} \\
    &= \frac{H_k(A)}{H_k(A) + \frac{1+\beta}{1+\alpha}T_k(A)}.
\end{split}
\end{equation}
By \cref{thm:relative_error} and~\cref{thm:asymptotic_robustness_exact}, uniform quantization noise significantly inflates the tail energy while inducing negligible relative error in the head energy, implying $\beta > \alpha$. Consequently, the scaling factor satisfies $\frac{1+\beta}{1+\alpha} > 1$. It follows that the denominator in the expression for $F_{\tilde{A}}(k)$ increases relative to the numerator, leading to:
\begin{equation}
\label{eq:concentration_inequality}
    F_{\tilde{A}}(k) < F_{A}(k), \quad \forall k.
\end{equation}Specifically,  if $\alpha=-1$, 
\begin{equation}
    F_{\tilde{A}}(k) = \frac{(1+\alpha)H_k(A)}{(1+\alpha)H_k(A) + (1+\beta)T_k(A)}=0, 
\end{equation}
which also leads to 
\begin{equation}
    F_{\tilde{A}}(k)< F_A(k), \quad \forall k.
\end{equation}
Note that as $\beta>\alpha$, $\beta>-1$. Finally, we observe that the stable rank $S_r(A) = \frac{\|A\|_F^2}{\|A\|_2^2}$ is mathematically equivalent to the reciprocal of the energy concentration at $k=1$:
\begin{equation}
    S_r(A) = \frac{\sum_{i} \sigma_i^2(A)}{\sigma_1^2(A)} = \frac{1}{F_A(1)}.
\end{equation}
Applying the inequality from Eq.~\eqref{eq:concentration_inequality} at $k=1$, we conclude:
\begin{equation}
    S_r(\tilde{A}) = \frac{1}{F_{\tilde{A}}(1)} > \frac{1}{F_{A}(1)} = S_r(A).
\end{equation}
\end{proof}

\section{More Experiments}
\label{sec:more_exp}
\subsection{ Empirical Validation of Quantization Error Unbiasedness}
\label{subsec:exp1_unbiasedness}

The validity of our spectral analysis relies heavily on Assumption~\ref{uniform_ass}, which posits that quantization noise acts as a centered random variable with zero mean ($\mathbb{E}[e_{ij}] = 0$). To empirically substantiate this theoretical premise, we analyze the distribution of quantization errors in a pre-trained LLM.

We employ a standard decoder-only Transformer architecture following the \textbf{GPT-2-124M} configuration ($L=12$ layers, $H=12$ heads, $d_{model}=768$). The model is trained on the OpenWebText dataset  using the AdamW optimizer with a peak learning rate of $1.5 \times 10^{-4}$ and a global batch size of 2,560 (achieved via gradient accumulation on 8 GPUs). The detailed hyperparameters are summarized in Table~\ref{tab:exp1_setup}. All \texttt{FP4} training in this work adopts a strict W4A4G4 quantization scheme, where weights, activations, and gradients are universally represented in the E2M1 \texttt{NVFP4} format. We note that due to the proprietary nature of NVIDIA's closed-source \texttt{FP4} training software stack, native hardware-supported \texttt{FP4} training is not currently accessible for open research. Consequently, consistent with standard practices in algorithmic research, our experiments with \texttt{NVFP4} are conducted through high-fidelity simulation in \texttt{BF16}. This simulation accurately emulates the quantization noise and dynamic range constraints of the E2M1 format while executing computations on standard \texttt{BF16}-capable hardware.  To verify the error properties, we extracted the quantization error matrices $E = Q(W) - W$ from the Multi-Layer Perceptron (MLP) weights across all 24 sub-layers of the model. Figure~\ref{fig:quant_error_dist} illustrates the error frequency histograms for each layer, overlaid with fitted normal distribution curves.
\begin{figure}[htbp]
    \centering
    \includegraphics[width=1.0\linewidth]{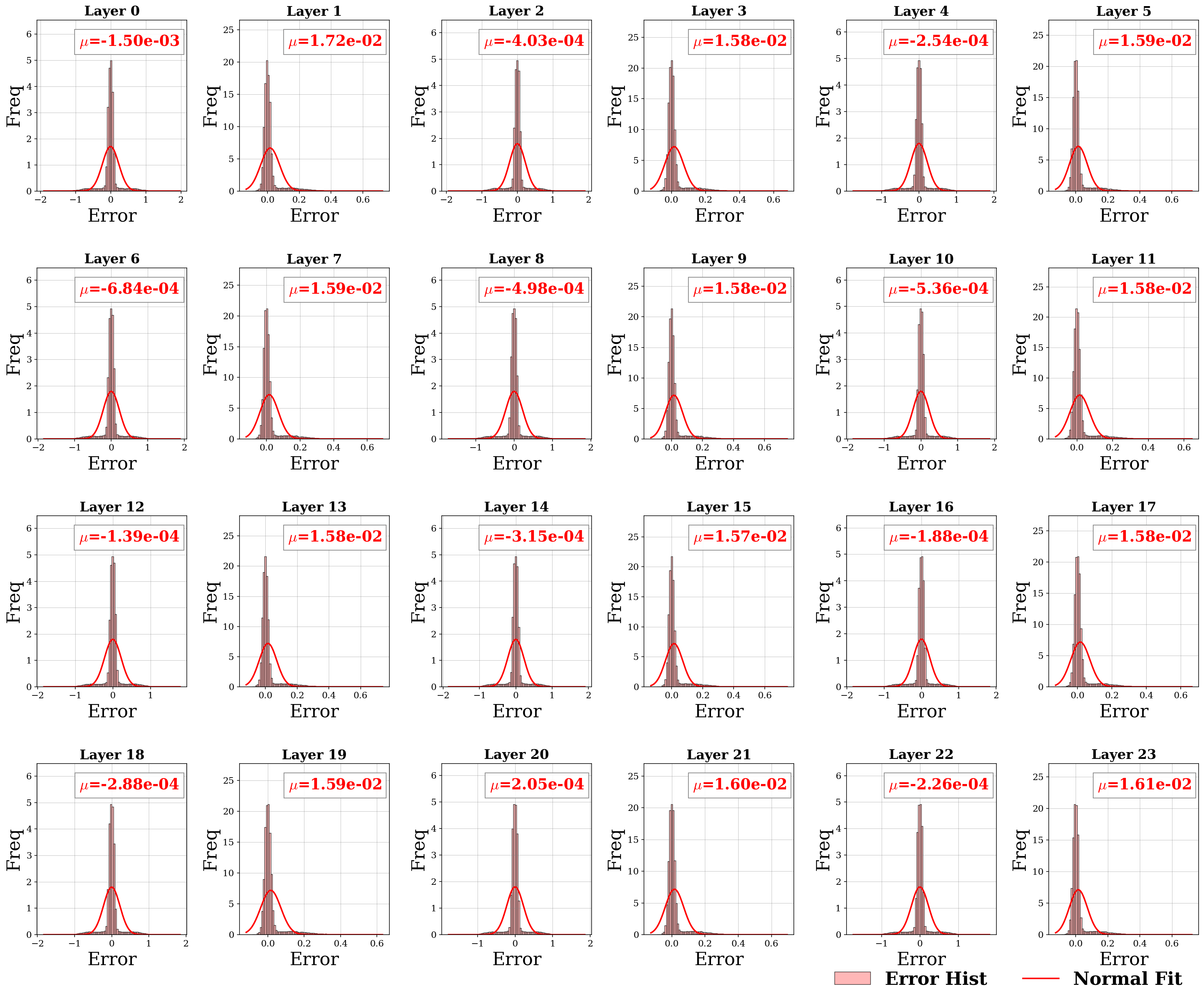} 
    \caption{ The histograms display the frequency of quantization errors for MLP linear layers in the \textbf{GPT-2-124M} model quantized with \texttt{NVFP4} . The red curves represent normal distribution fits. The annotated $\mu$ values indicate the empirical mean of the error for each layer.}
    \label{fig:quant_error_dist}
\end{figure}
\begin{table*}[htbp]
\centering
\caption{ This table summarizes the model architecture, optimization hyperparameters, and quantization settings used to validate the unbiasedness of quantization error in Section~\ref{subsec:exp1_unbiasedness}.}
\label{tab:exp1_setup}
\vspace{8pt}
\begin{tabular}{l|l||l|l}
\hline
\multicolumn{2}{c||}{\textbf{Model \& Data}} & \multicolumn{2}{c}{\textbf{Training Hyperparameters}} \\ \hline
\textbf{Architecture} & GPT-2 (Decoder-only) & \textbf{Optimizer} & AdamW \\
\textbf{Parameters} & $\approx$ 124M & \textbf{Peak Learning Rate} & $1.5 \times 10^{-4}$ \\
\textbf{Layers ($L$)} & 12 & \textbf{Weight Decay} & 0.1 \\
\textbf{Hidden Size ($d$)} & 768 & \textbf{Global Batch Size} & 2,560 \\
\textbf{Attention Heads ($H$)} & 12 & \textbf{Gradient Accumulation} & 40 steps \\
\textbf{Context Window} & 1024 tokens & \textbf{Warmup Steps} & 20 \\
\textbf{Dataset} & OpenWebText & \textbf{Gradient Clipping} & 1.0 \\
\textbf{Hardware} & 8 $\times$ GPUs & \textbf{Epochs} & 1 \\ \hline
\multicolumn{4}{c}{\textbf{Quantization Configuration (\texttt{NVFP4})}} \\ \hline
\textbf{Format} & \multicolumn{3}{l}{\texttt{NVFP4} (E2M1) - 4-bit Floating Point (2 exponent bits, 1 mantissa bit)} \\
\textbf{Scope} & \multicolumn{3}{l}{Weights, Activations, Gradients (Full \texttt{FP4} Training)} \\
\textbf{Scheme} & \multicolumn{3}{l}{Block-wise Symmetric Uniform Quantization} \\ \hline
\end{tabular}
\end{table*}

\paragraph{Results.} As evidenced in Figure~\ref{fig:quant_error_dist}, the quantization error distributions exhibit a high degree of symmetry around zero. Crucially, the empirical mean values ($\mu$) for all layers are negligible, ranging from order $10^{-4}$ to $10^{-2}$ (e.g., Layer 0: $\mu = -1.50 \times 10^{-3}$, Layer 10: $\mu = -5.36 \times 10^{-4}$), which are statistically insignificant compared to the dynamic range of the weights. This empirical evidence strongly supports Assumption~\ref{uniform_ass}. While the theoretical assumption simplifies the noise to a uniform distribution, the critical property required for our spectral perturbation bounds is the unbiased nature of the noise (i.e., the first moment $\mathbb{E}[E] = 0$). The observed distributions confirm that the \texttt{NVFP4} quantization mechanism does not introduce a systematic bias to the weight matrices, thereby validating the use of centered random matrix theory in our subsequent spectral analysis.
\subsection{ Validation of Relative Error Bounds}
\label{subsec:exp2_relative_error}

Building on the confirmation of error unbiasedness, this experiment aims to rigorously validate Theorem~\ref{thm:relative_error}, which posits a fundamental spectral asymmetry: the relative quantization error $\epsilon_k$ is strictly inversely proportional to the magnitude of the singular value $\sigma_k$. Mathematically, this relationship is expressed as $\epsilon_k \propto 1/\sigma_k(A)$. This experiment inherits the identical experimental setup described in Section~\ref{subsec:exp1_unbiasedness} and summarized in Table~\ref{tab:exp1_setup}. We utilize the same pre-trained \textbf{GPT-2-124M} model and apply the simulated \texttt{NVFP4} (W4A4G4) quantization to all linear layers. For each layer $l$, we compute the Singular Value Decomposition (SVD) of the original parameter matrix $A^{(l)}$ and its quantized counterpart $\tilde{A}^{(l)}$. We then calculate the component-wise relative error $\epsilon_k$ for each rank index $k$:
\begin{equation}
    \epsilon_k = \frac{|\sigma_k(\tilde{A}^{(l)}) - \sigma_k(A^{(l)})|}{\sigma_k(A^{(l)})}
\end{equation}
To test the linearity predicted by Theorem~\ref{thm:relative_error}, we perform a linear regression analysis of $\epsilon_k$ against the inverse singular values $1/\sigma_k(A^{(l)})$.

\begin{figure*}[htbp]
    \centering
    \includegraphics[width=1.0\linewidth]{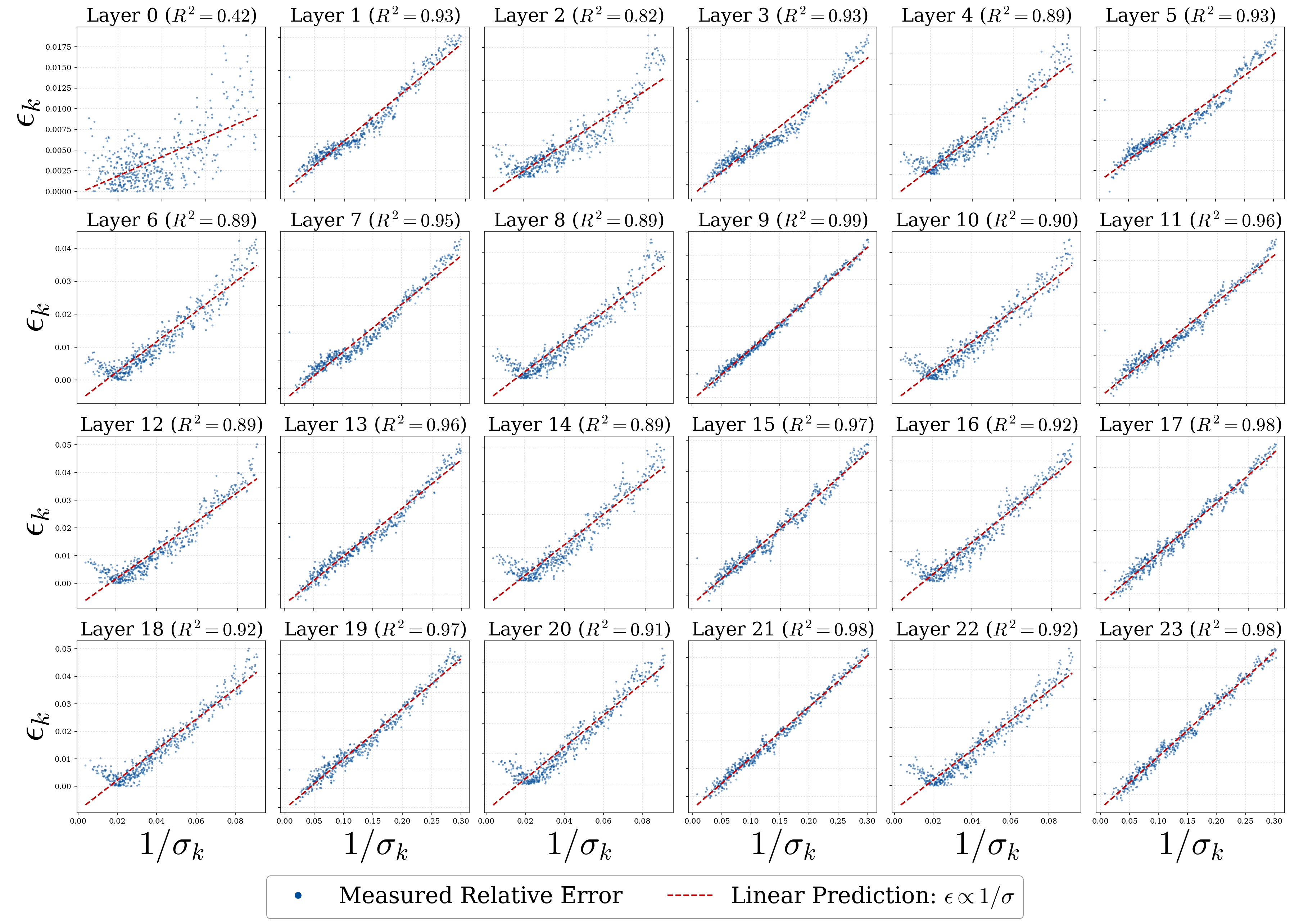} % 请确保文件名与 Figure 4 对应
    \caption{ The scatter plots illustrate the relationship between the relative quantization error $\epsilon_k$ (y-axis) and the inverse of the singular value $1/\sigma_k$ (x-axis) across varying layers of the \textbf{GPT-2-124M} model. The red dashed lines represent the linear regression fit. The coefficient of determination ($R^2$) is annotated for each layer, indicating the goodness of fit.}
    \label{fig:relative_error_analysis}
\end{figure*}
\paragraph{Results.} The results are presented in Figure~\ref{fig:relative_error_analysis}, covering layers 0 through 23. The visual evidence strongly corroborates our theoretical derivation. Across the majority of layers, we observe a striking linear relationship between the relative error $\epsilon_k$ and the reciprocal of the singular value $1/\sigma_k$. This confirms the prediction of~\cref{eq:inverse} that $\epsilon_k(\theta) \propto \frac{1}{\sigma_k(A)}$. The linear regression fits yield high coefficients of determination ($R^2$). For instance, deeper layers such as Layer 9 ($R^2=0.99$), Layer 19 ($R^2=0.97$), and Layer 23 ($R^2=0.98$) exhibit near-perfect adherence to the theoretical model. While some initial layers (e.g., Layer 0 with $R^2=0.42$) show higher variance due to distinct initialization distributions, the global trend remains consistent. The data demonstrates that dominant singular values (where $1/\sigma_k$ is small) suffer from negligible relative error, effectively preserving the head information. Conversely, as we move to the tail (where $1/\sigma_k$ is large), the relative error scales linearly, leading to significant signal distortion. This mechanism explains the spectral flattening phenomenon: the quantization noise floor disproportionately destroys the low-energy components responsible for fine-grained semantic distinctions, while leaving the coarse-grained structural information relatively intact.
This experiment provides robust empirical support for Theorem~\ref{thm:relative_error}, confirming that uniform quantization does not degrade information uniformly; rather, it selectively erodes the heavy-tailed spectral structure essential for LLM performance.

\subsection{ Universality of Stable Rank Increase}
\label{subsec:exp3_stablerank}

Theorem~\ref{thm:stable_rank_impact} predicts that uniform quantization inherently flattens the spectral distribution, leading to a strict increase in the stable rank of parameter matrices ($S_r(\tilde{A}) > S_r(A)$). To demonstrate the universality of this phenomenon, we conducted full \texttt{FP4} training (simulated W4A4G4) on a diverse set of Large Language Models. The experimental suite includes \textbf{GPT-2-124M} and five state-of-the-art architectures: \textbf{TinyLlama-1.1B}, \textbf{Phi-1.5-1.3B}, \textbf{Pythia-1.4B}, \textbf{DeepSeek-Coder-1.3B}, and \textbf{Qwen2.5-1.5B}. These models cover a wide spectrum of design choices, including varying depths (12 to 28 layers), embedding dimensions (768 to 2048), head counts (12 to 32), and vocabulary sizes (32k to 152k). The detailed architectural specifications and training hyperparameters are summarized in Table~\ref{tab:exp3_models}. All models were trained on the OpenWebText dataset using 8 GPUs with \texttt{NVFP4} quantization. We compared the singular value spectra of the pre-quantization weights (\texttt{BF16}) and their post-quantization counterparts (\texttt{NVFP4}) across all layers for each model. The results are visualized in Figure~\ref{fig:exp3_part1}, Figure~\ref{fig:exp3_part2}, and Figure~\ref{fig:exp3_part3}.

\begin{table*}[t]
\centering
\caption{Model Architectures and Training Hyperparameters for Section \ref{subsec:exp3_stablerank}. This table details the diverse set of LLMs used to validate the universality of Theorem~\ref{thm:stable_rank_impact}. All models utilize AdamW optimizer, Cosine learning rate schedule, and W4A4G4 (\texttt{NVFP4}) quantization.}   
\label{tab:exp3_models}
\vspace{8pt}
\resizebox{\textwidth}{!}{%
\begin{tabular}{l|c|c|c|c|c|c|c|c}
\hline
\textbf{Model Family} & \textbf{Params} & \textbf{Layers} & \textbf{Hidden Dim} & \textbf{Heads} & \textbf{KV Heads} & \textbf{Vocab Size} & \textbf{Peak LR} & \textbf{Context Window} \\ \hline
\textbf{GPT-2}  & 124M & 12 & 768 & 12 & 12 & 50,257 & $1.5 \times 10^{-4}$ & 1024 \\ 
\textbf{TinyLlama} & 1.1B & 22 & 2048 & 32 & 8 & 32,000 & $4.0 \times 10^{-3}$ & 1024 \\
\textbf{Phi-1.5} & 1.3B & 24 & 2048 & 32 & 32 & 51,200 & $4.0 \times 10^{-4}$ & 1024 \\
\textbf{Pythia} & 1.4B & 24 & 2048 & 16 & 16 & 50,304 & $4.0 \times 10^{-4}$ & 1024 \\
\textbf{DeepSeek-Coder} & 1.3B & 24 & 2048 & 32 & 32 & 50,304 & $4.0 \times 10^{-4}$ & 1024 \\
\textbf{Qwen2.5} & 1.5B & 28 & 1536 & 12 & 2 & 151,936 & $4.0 \times 10^{-4}$ & 1024 \\ \hline
\end{tabular}%
}
\end{table*}

% Figure Part 1: GPT-2 & TinyLlama
\begin{figure*}[htbp]
    \centering
    \begin{minipage}{0.49\textwidth}
        \centering
        \includegraphics[width=\linewidth]{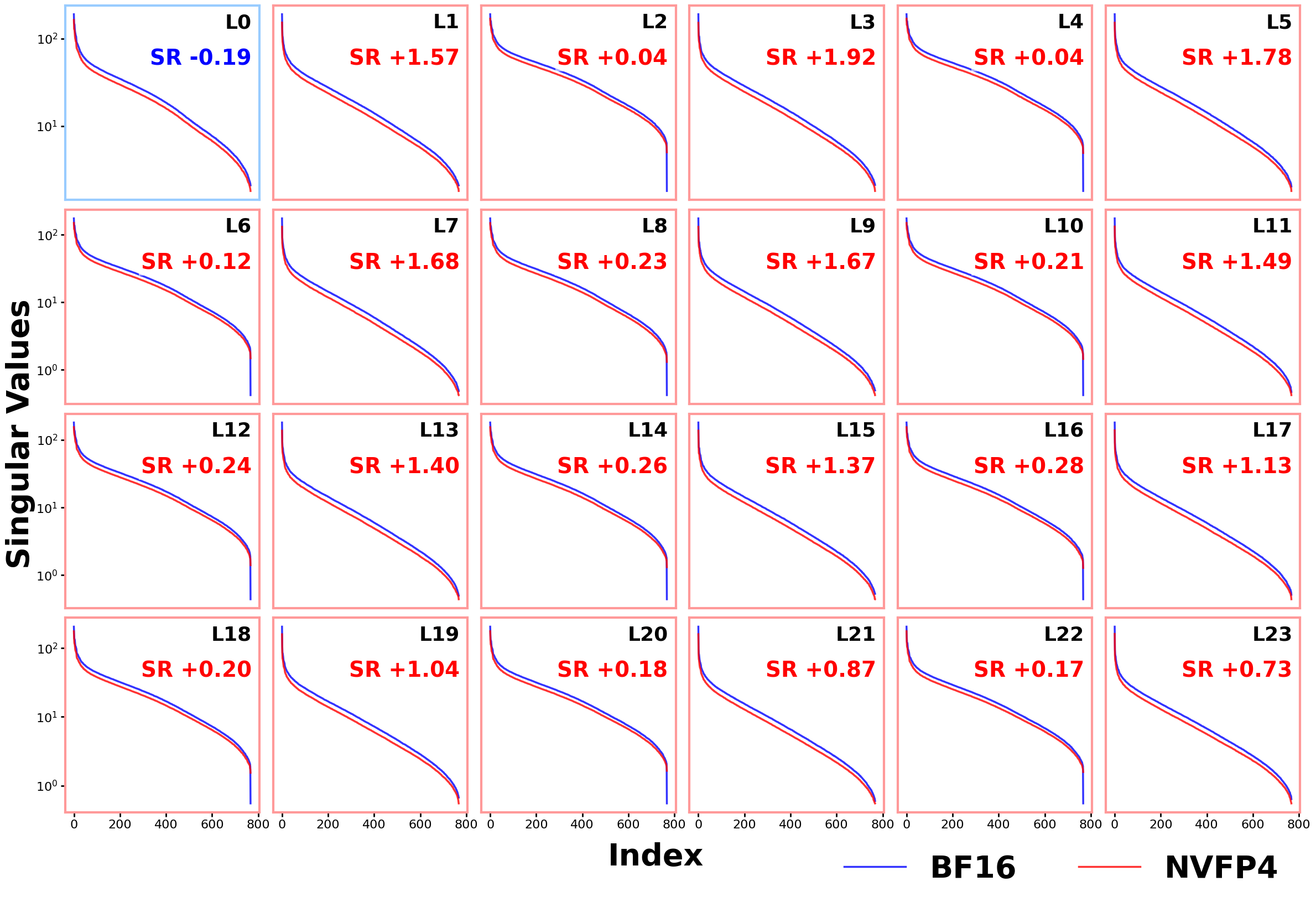}
        \caption*{\textbf{ GPT-2-124M}}
    \end{minipage}
    \hfill
    \begin{minipage}{0.49\textwidth}
        \centering
        \includegraphics[width=\linewidth]{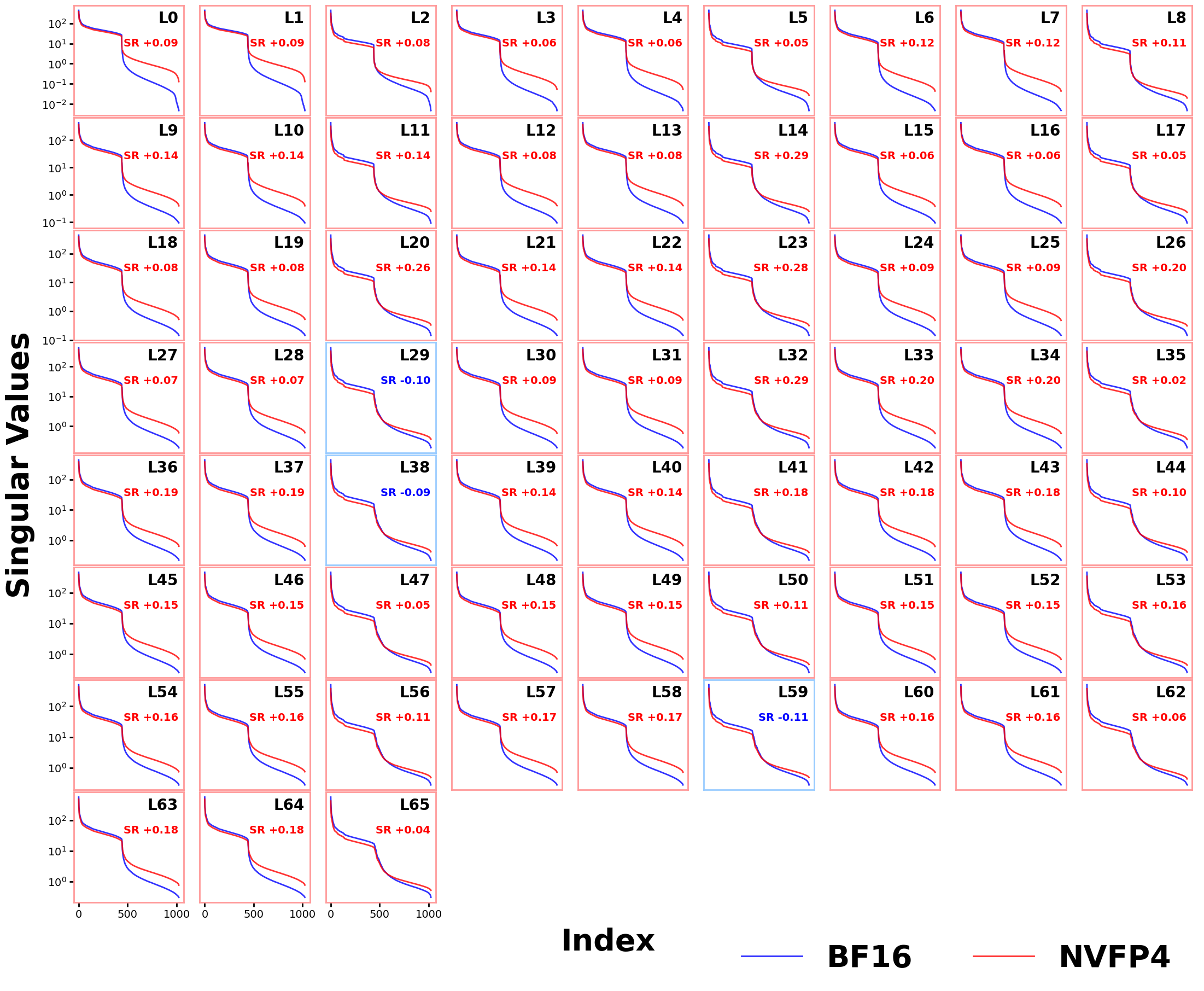}
        \caption*{\textbf{ TinyLlama-1.1B}}
    \end{minipage}
    \caption{Spectral analysis of \textbf{GPT-2-124M} and \textbf{TinyLlama-1.1B}, where quantization induces a consistent positive shift in Stable Rank (SR+) in the majority of layers.}
    \label{fig:exp3_part1}
\end{figure*}

% Figure Part 2: Phi-1.5 & Pythia

% Figure Part 3: DeepSeek & Qwen

\paragraph{Results.} As shown in Figures \ref{fig:exp3_part1}-\ref{fig:exp3_part3}, every tested architecture exhibits the characteristic spectral flattening pattern. The head singular values remain relatively stable, while the tail singular values are significantly elevated by the quantization noise floor. Quantitatively, this results in a strict increase in the Stable Rank for the vast majority of layers (indicated by SR+ in the plots). Starting with the \textbf{GPT-2-124M}, we observe the foundational spectral degradation. This trend is amplified in larger, deeper models. In \textbf{Qwen2.5-1.5B}, which features a deep narrow structure (28 layers, 1536 dim), the stable rank increase is pronounced across middle layers. Similarly, \textbf{TinyLlama-1.1B} and \textbf{Phi-1.5-1.3B} show consistent spectral degradation. This cross-model consistency validates Theorem~\ref{thm:stable_rank_impact} as a fundamental law governing low-precision embeddings, suggesting that the loss of representational capacity via tail truncation is an unavoidable consequence of uniform quantization in heavy-tailed distributions.
\begin{figure*}[htbp]
    \centering
    \begin{minipage}{0.49\textwidth}
        \centering
        \includegraphics[width=\linewidth]{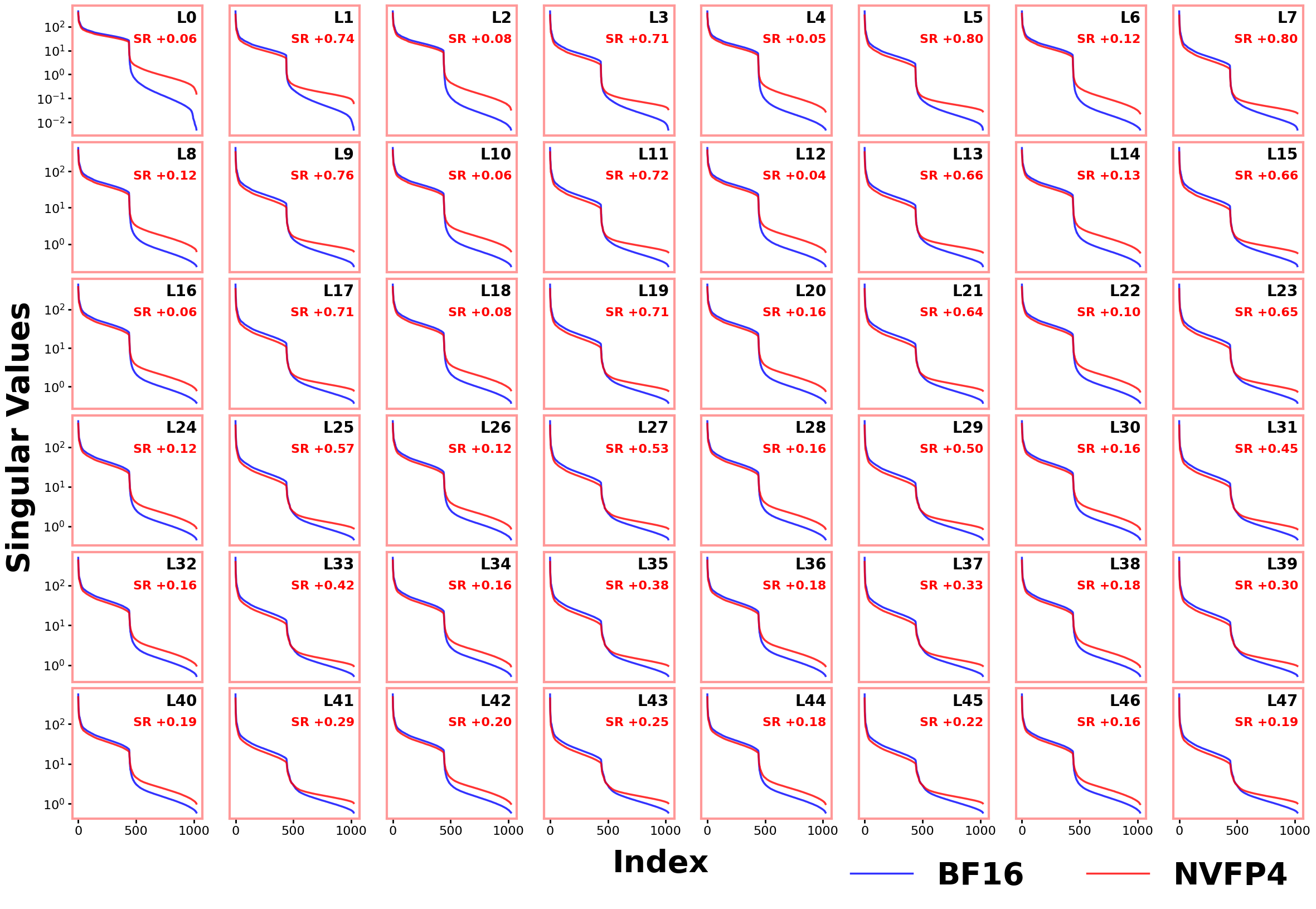}
        \caption*{\textbf{ Phi-1.5-1.3B}}
    \end{minipage}
    \hfill
    \begin{minipage}{0.49\textwidth}
        \centering
        \includegraphics[width=\linewidth]{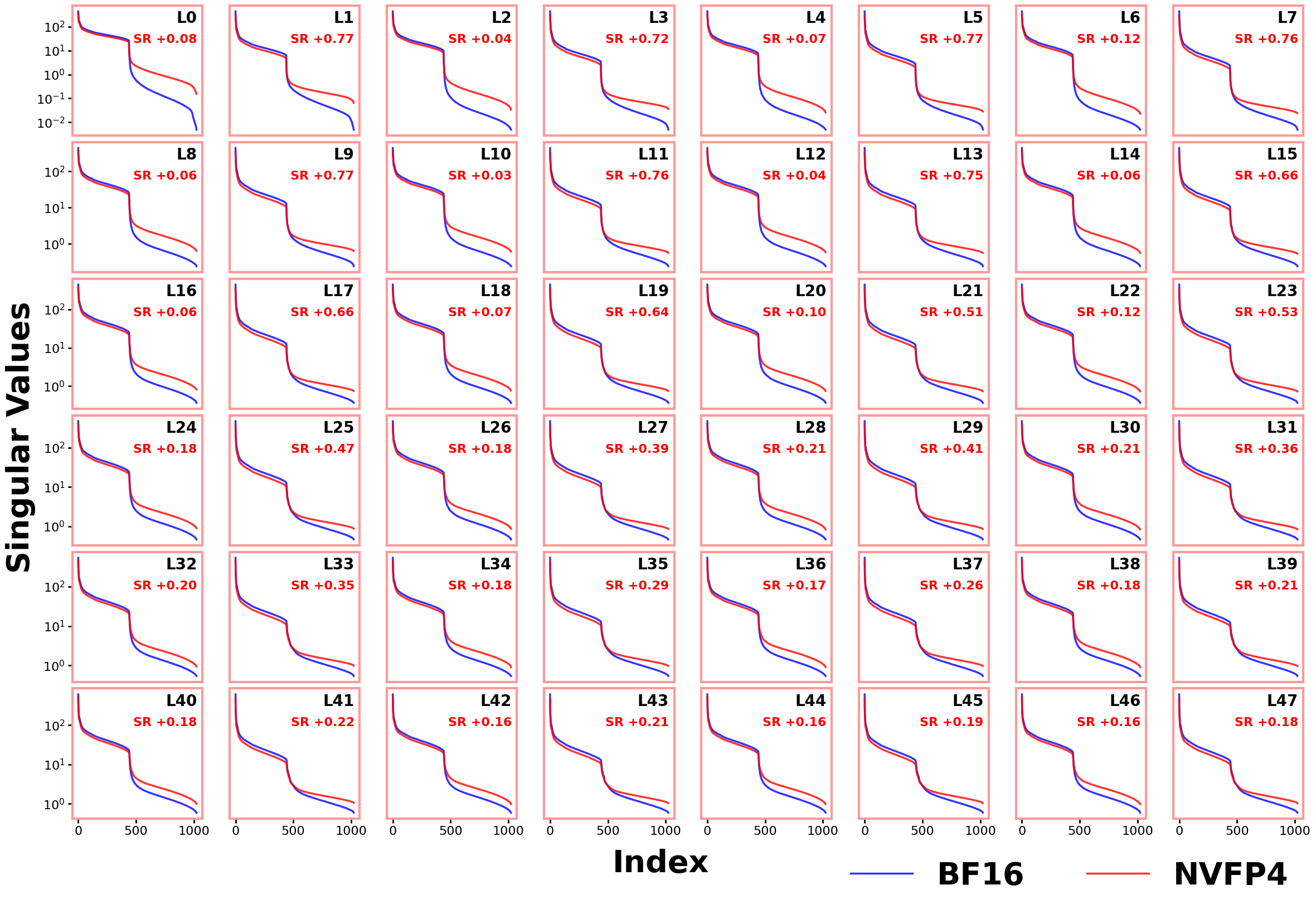}
        \caption*{\textbf{ Pythia-1.4B}}
    \end{minipage}
    \caption{Spectral analysis of \textbf{Phi-1.5-1.3B} and \textbf{Pythia-1.4B}, where quantization induces a consistent positive shift in Stable Rank (SR+) in the majority of layers.}
    \label{fig:exp3_part2}
\end{figure*}
\begin{figure*}[htbp]
    \centering
    \begin{minipage}{0.49\textwidth}
        \centering
        \includegraphics[width=\linewidth]{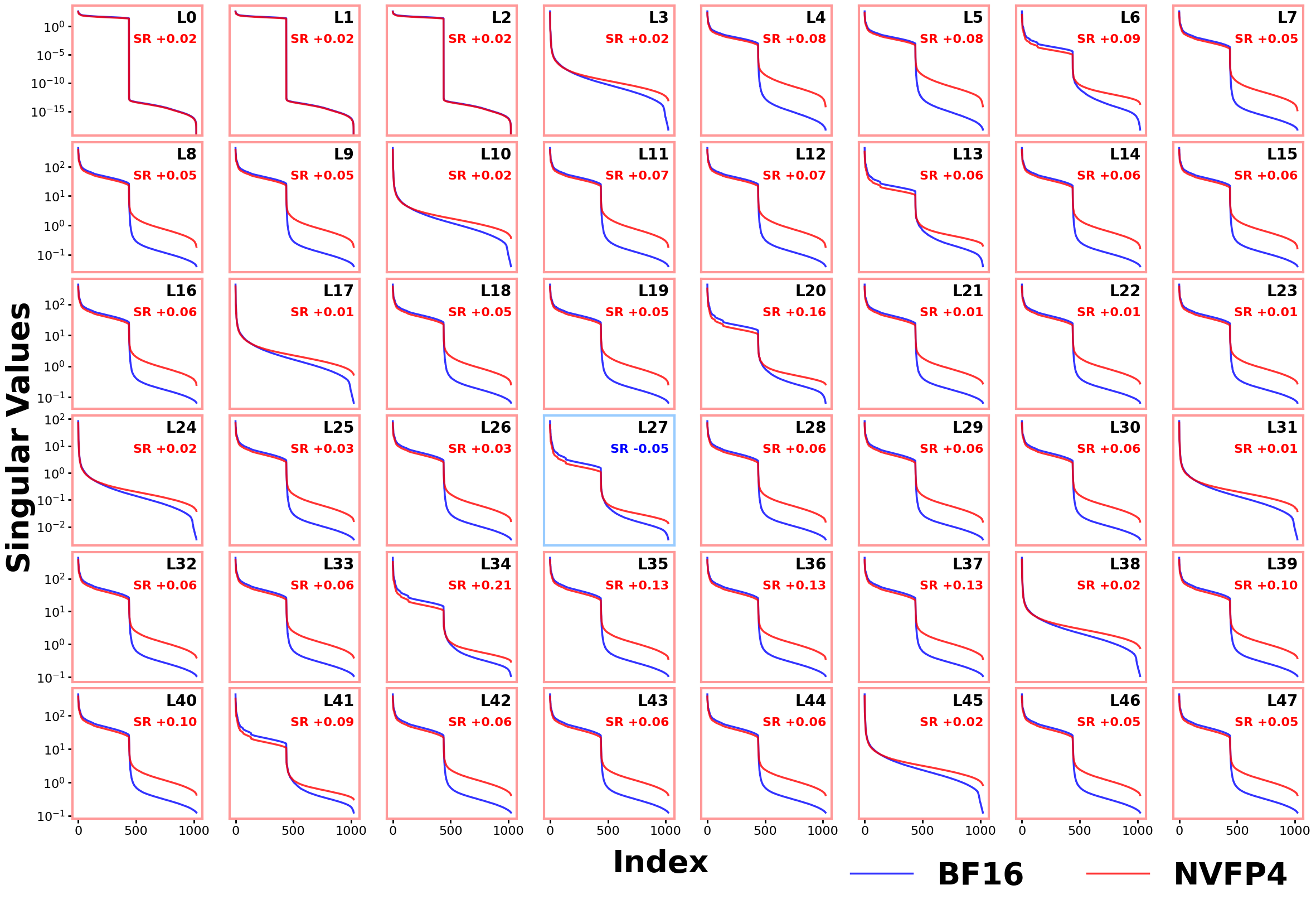}
        \caption*{\textbf{ DeepSeek-Coder-1.3B}}
    \end{minipage}
    \hfill
    \begin{minipage}{0.49\textwidth}
        \centering
        \includegraphics[width=\linewidth]{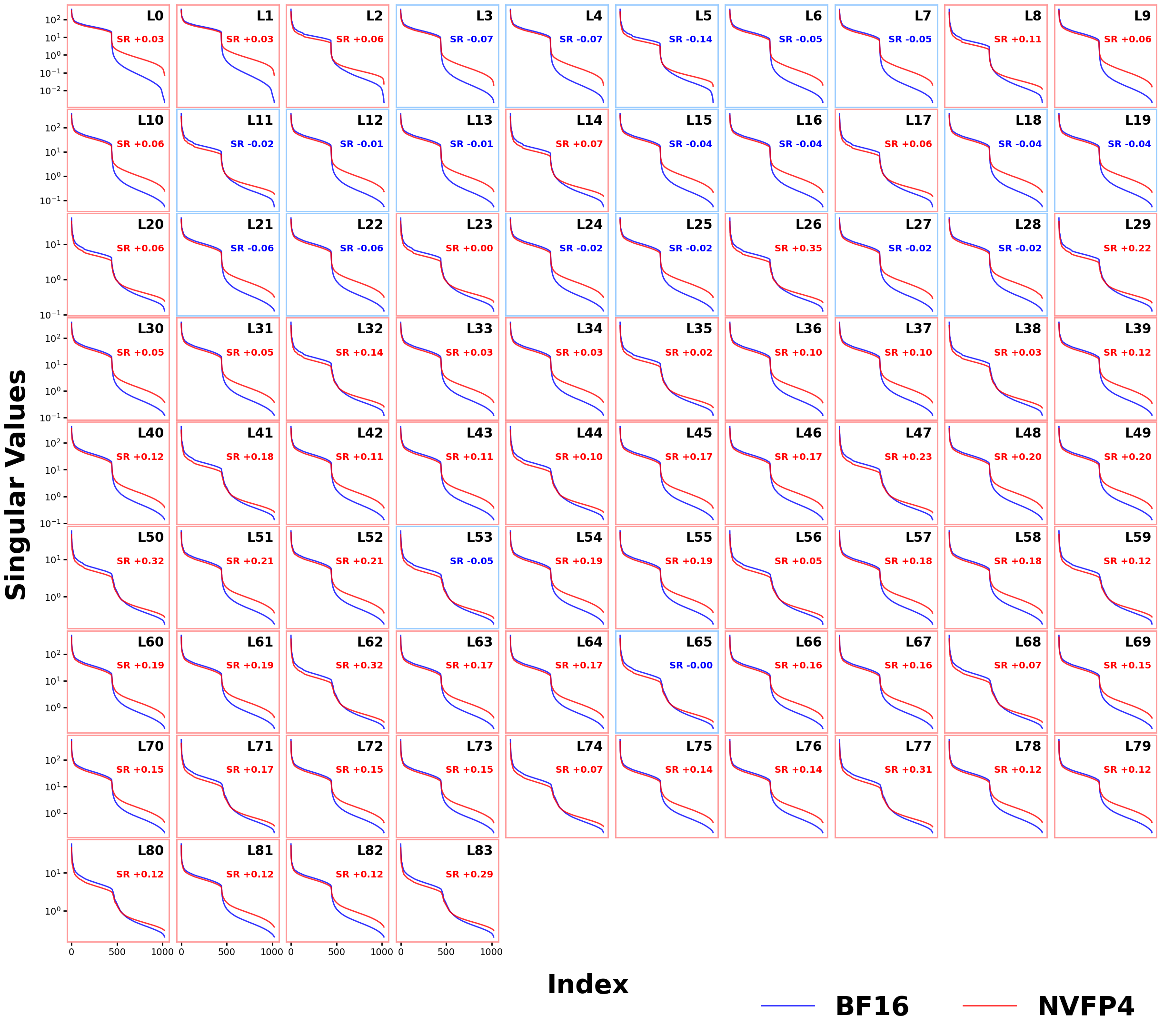}
        \caption*{\textbf{ Qwen2.5-1.5B}}
    \end{minipage}
    \caption{Spectral analysis of \textbf{DeepSeek-Coder-1.3B} and \textbf{Qwen2.5-1.5B}, where quantization induces a consistent positive shift in Stable Rank (SR+) in the majority of layers.}
    \label{fig:exp3_part3}
\end{figure*}

\end{document}